\documentclass[lettersize,journal]{IEEEtran}

\usepackage{amsmath,amssymb,amsfonts,amsthm}
\usepackage{algorithmic}
\usepackage{algorithm}
\usepackage{array}
\usepackage{textcomp}
\usepackage{stfloats}
\usepackage{url}
\usepackage{verbatim}
\usepackage{graphicx}
\usepackage{booktabs}
\usepackage{multicol,multirow}
\usepackage{tabularx}
\usepackage{subfigure}
\usepackage[table,xcdraw]{xcolor}
\usepackage{float}
\usepackage[numbers,sort&compress]{natbib}

\newcommand{\eg}{\textit{e}.\textit{g}., }

\newcommand{\Rmnum}[1]{\uppercase\expandafter{\romannumeral #1}}
\newcommand{\hhs}{\text{H}_{2}\text{S}}
\newcommand{\soo}{\text{SO}_{2}}
\newcommand{\rr}{\text{R}^{2}}

\newtheorem{thm}{Theorem}[section]

\newtheorem{lem}{Lemma}[section]

\begin{document}

\title{TMoE-P: Towards the Pareto Optimum \\for Multivariate Soft Sensors}

\author{Licheng Pan, ~\IEEEmembership{Student Member,~IEEE}, Hao Wang, ~\IEEEmembership{Student Member,~IEEE}, Zhichao Chen, Yuxing Huang, Xinggao Liu 
\thanks{The authors are with the State Key Laboratory of Industrial 
Control Technology, College of Control Science and Engineering, Zhejiang
University, Hangzhou 310027, China (e-mail: 22132045@zju.edu.cn;
haohaow@zju.edu.cn; 12032042@zju.edu.cn; 22132047@zju.edu.cn; lxg@zju.edu.cn).}}


\maketitle

\begin{abstract}
Multi-variate soft sensor seeks accurate estimation of multiple quality variables using measurable process variables, which have emerged as a key factor in improving the quality of industrial manufacturing.
The current progress stays in some direct applications of multitask network architectures; however, there are two fundamental issues remain yet to be investigated with these approaches: (1) negative transfer, where sharing representations despite the difference of discriminate representations for different objectives degrades performance; (2) seesaw phenomenon, where the optimizer focuses on one dominant yet simple objective at the expense of others.
In this study, we reformulate the multi-variate soft sensor to a multi-objective problem, to address both issues and advance state-of-the-art performance.
To handle the negative transfer issue, we first propose an Objective-aware Mixture-of-Experts (OMoE) module, utilizing objective-specific and objective-shared experts for parameter sharing while maintaining the distinction between objectives.
To address the seesaw phenomenon, we then propose a Pareto Objective Routing (POR) module, adjusting the weights of learning objectives dynamically to achieve the Pareto optimum, with solid theoretical supports. 
We further present a Task-aware Mixture-of-Experts framework for achieving the Pareto optimum (TMoE-P) in multi-variate soft sensor, which consists of a stacked OMoE module and a POR module. 
We illustrate the efficacy of TMoE-P with an open soft sensor benchmark, where TMoE-P effectively alleviates the negative transfer and seesaw issues and outperforms the baseline models.
\end{abstract}


\section{Introduction}
\label{sec:introduction}

\IEEEPARstart{P}{rocess} monitoring plays a significant role in contemporary industry and is directly associated to the manufacture of critical industrial products, \eg oil, gas, rare metals, iron, and steel, which are integral to modern human life and national economies. 
Monitoring the dynamics of critical quality variables in manufacturing has become one of the main concerns in order to meet urgent and demanding requirements, including increasing yields, reducing material consumption, protecting the environment and ensuring the safety of manufacturing processes.
The main challenge with monitoring is that some quality variables are extremely difficult to measure.
For instance, in deep water gas-lift oil well process, down-hole pressure is an extremely useful indicator for assessing the manufacture quality~\citep{gas-lift}, but it is difficult to measure with hardware sensors (\eg permanent down-hole pressure gauges) due to the inability of hardware sensors in high pressure and salinity environments~\citep{gauges}.

Soft sensors aims to estimate immeasurable quality variables with measurable process variables, which can be categorized as model-driven and data-driven approaches. 
With the success of machine learning and database technology, data-driven approaches have been predominant for building effective soft sensors, which involves building statistical estimates of quality variables with process variables.
At the very beginning, the data-driven soft sensors were implemented with linear statistical approaches, such as principal component regression~\citep{pcr1,pcr2}, partial least squares~\citep{pls} and gaussian process~\citep{gaussian}.
To depict nonlinear relationships between variables, nonlinear models in the machine learning community were further involved in soft sensors, represented by the support vector regression~\citep{svr1,svr2}, decision tree~\citep{decision-tree} and kernel regression~\citep{kernel-reg}.
More recently, with the great advancement of deep learning techniques~\citep{wangescm,fanlearnable}, soft sensors have been dominated by deep neural methods. Representative methods can be roughly categorized into several types: autoencoder networks~\citep{ae1,ae2,ae3}, recurrent neural networks~\citep{rnn1,rnn2,rnn3}, convolution neural networks~\citep{cnn1,cnn2}, graph neural networks~\citep{gnn1,gnn2}, and self-attentive networks~\citep{self-att1,self-att2}, with each type of methods its own strengths and weaknesses. For example, Autoencoder-based sensors can be refined with semi-supervised setting, but they struggle to capture long-term sequence patterns; RNN-based sensors can be incrementally updated for real-time monitoring, but they suffer from sub-optimal accuracy and single-threaded computational paradigm; self-attention based sensors can process the whole sequence in a parallel manner, but they suffer from huge computational cost and overfitting risk. 

Overall, the aforementioned line of research focuses on developing more efficient and effective architectures with the aim of improving the estimation accuracy of \emph{single} quality variable.
Despite their success, in real-world practice there is more than one quality variable to be estimated, which requires the construction of an effective multi-variate soft sensor (MVSS). 
For example, the product concentration and reactant concentration in the reactive distillation process must be measured concurrently to track the separation energy consumption and product purity.

Different from Multi-task Learning (MTL), Multi-objective Optimization (MOO) is a technology that can coordinate multiple objectives with internal conflicts and connections to obtain Pareto optimal solutions, mainly through the sharing mechanism. Generally, parameter sharing mechanism can be separated into two categories: hard parameter sharing and soft parameter sharing. Hard parameter sharing refers to the uniform sharing of the network's underlying parameters while maintaining the network's top-level parameters' independence for objectives, such as UberNet~\citep{ubernet}, multilinear relationship networks~\citep{mrn}, and stochastic filter groups~\citep{sfg}. However, in circumstances when the connection between objectives is poor, the direct sharing of underlying parameters will result in performance decrease, which is known as NT phenomenon. Soft parameter sharing introduces distinctive underlying shared expert parameters to each objective, such as cross-stitch networks~\citep{cross-stitch}, sluice networks~\citep{sluice}, and MMoE~\citep{mmoe,bmoe} The involvement of experts helps mitigate the NT issue, but because the variations and interactions among experts are neglected, it is challenging for multiple objectives to reach the same performance of a single objective at once, leading to the seesaw phenomenon.

This paper proposes a Task-aware Mixture-of-Experts framework for achieving the Pareto optimum (TMoE-P) to solve issues of NT and seesaw, which were brought on by disregarding the commonality and variability across experts in parameter sharing. The model is made up of two modules: Objective-aware Mixture-of-Experts (OMoE) and Pareto Objective Routing (POR). To start, the OMoE module explicitly distinguishes between objective-specific and objective-shared experts, reducing potentially damaging interactions among representations. Secondly, the POR module solves an optimization problem concerning the gradient of shared parameters using the gradient of each objective, and then obtains the Pareto optimal model parameters, which alleviates some objectives' performance sacrifice throughout the model training process.

The main contributions of this paper are summarized as follows:

\begin{enumerate}
\item The soft sensor problem is stated from the standpoint of MOO, and it is highlighted that the regression problem of multiple quality variables is transformed into the joint optimization of multiple regression objectives. A MOO model will be used to model the relationships between and within sequences of quality variables.
\item A TMoE-P model based on the OMoE module and POR module is proposed to solve the NT and seesaw issues of multi-objective soft sensor problem. The solutions for the two modules are explicit objective-specific with objective-shared representations learning, and Pareto optimal MOO.
\item Numerous off-line experiments were conducted in the Sulfur Recovery Unit process to assess the efficiency of TMoE-P in resolving the issues of NT and seesaw. The results and mathematical proofs reveal that TMoE-P outperforms the baseline models in soft sensor applications.
\end{enumerate}

The rest of this paper is organized as follows: Section \Rmnum{2} covers the preliminaries of multi-objective soft sensor problem and Pareto theory. Section \Rmnum{3} introduced the novel TMoE-P approach. Section \Rmnum{4} is an experimental study on the well-known Sulfur Recovery Unit. The final section is the conclusion.
\section{Preliminaries}
\label{sec:preliminaries}
\subsection{Multi-Objective Soft Sensor Problem}
Given a dataset with i.i.d. data points $\left\{ \boldsymbol{x}_{i}, \boldsymbol{y}_{i} \right\}_{i \in \{1,2,\ldots,N\}}$ where $\boldsymbol{x}_{i}=\left[ x_{i, 1}, x_{i, 2}, \ldots, x_{i, D}\right]^{\top}$ is the $D$-dimensional process variables, $\boldsymbol{y}_{i}=\left[ y_{i, 1}, y_{i, 2}, \ldots, y_{i, K}\right]^{\top}$ is the $K$-dimensional quality variables, $y_{i, k}$ is the label of the $k^{\text{th}}$ objective, $K$ is the number of objectives, and $N$ is the total number of data points. Multi-objective soft sensor (MOSS) seeks to establish an inference mathematical model $f \left( \boldsymbol{x} | \boldsymbol{\theta} \right)$ from process variables to quality variables through sharing mechanism. When a new process variables sequence $\boldsymbol{x}_{N+m} |_{m \geq 1}$ arrives, the quality variables sequence $\hat{\boldsymbol{y}}_{N+m}$ can be predicted as follows:
\begin{equation}\label{eq:pred}
    \hat{\boldsymbol{y}}_{N+m} = f(\boldsymbol{x}_{N+m} | \boldsymbol{\theta}) \cong \left[ y_{N+m, 1}, \ldots, y_{N+m, K}\right]^{\top}.
\end{equation}

Once we have set the associated loss function for each objective as $\mathcal{L}_{k}$, the mathematical form of MOSS is given as:
\begin{equation}\label{eq:multi-loss}
\begin{aligned}
    &\min_{\boldsymbol{\theta}_{\text{sh}}, \boldsymbol{\theta}_{1}, \ldots, \boldsymbol{\theta}_{K}} \mathcal{L}\left(\boldsymbol{\theta}_{\text{sh}}, \boldsymbol{\theta}_{1}, \ldots, \boldsymbol{\theta}_{K}\right) \\
    =&\min_{\boldsymbol{\theta}_{\text{sh}}, \boldsymbol{\theta}_{1}, \ldots, \boldsymbol{\theta}_{K}}\left[\hat{\mathcal{L}}_{1}\left(\boldsymbol{\theta}_{\text{sh}}, \boldsymbol{\theta}_{1}\right), \ldots, \hat{\mathcal{L}}_{L}\left(\boldsymbol{\theta}_{\text{sh}}, \boldsymbol{\theta}_{K}\right)\right]^{\top}.
\end{aligned}
\end{equation}
where $\boldsymbol{\theta}_{\text{sh}}$ are shared parameters between objectives, $\boldsymbol{\theta}_{k}$ are objective-specific, $\hat{\mathcal{L}}_{k}\left(\boldsymbol{\theta}_{\text{sh}}, \boldsymbol{\theta}_{k}\right)$ is the empirical loss of $k^{\text{th}}$ objective on the dataset, defined as $\hat{\mathcal{L}}_{k}\left(\boldsymbol{\theta}_{\text{sh}}, \boldsymbol{\theta}_{k}\right)=\frac{1}{N} \sum_{i=1}^{N} \mathcal{L}_{k}\left(f_{k}\left(\boldsymbol{x}_{i}|\boldsymbol{\theta}_{\text{sh}}, \boldsymbol{\theta}_{k}\right), y_{i, k}\right)$ and $f_{k}$ is the inference submodel in $f$ corresponding to $k^{\text{th}}$ objective.

\begin{figure}[!t]
    \centering
    \subfigure[Hard Sharing]{\includegraphics[width=0.32\columnwidth]{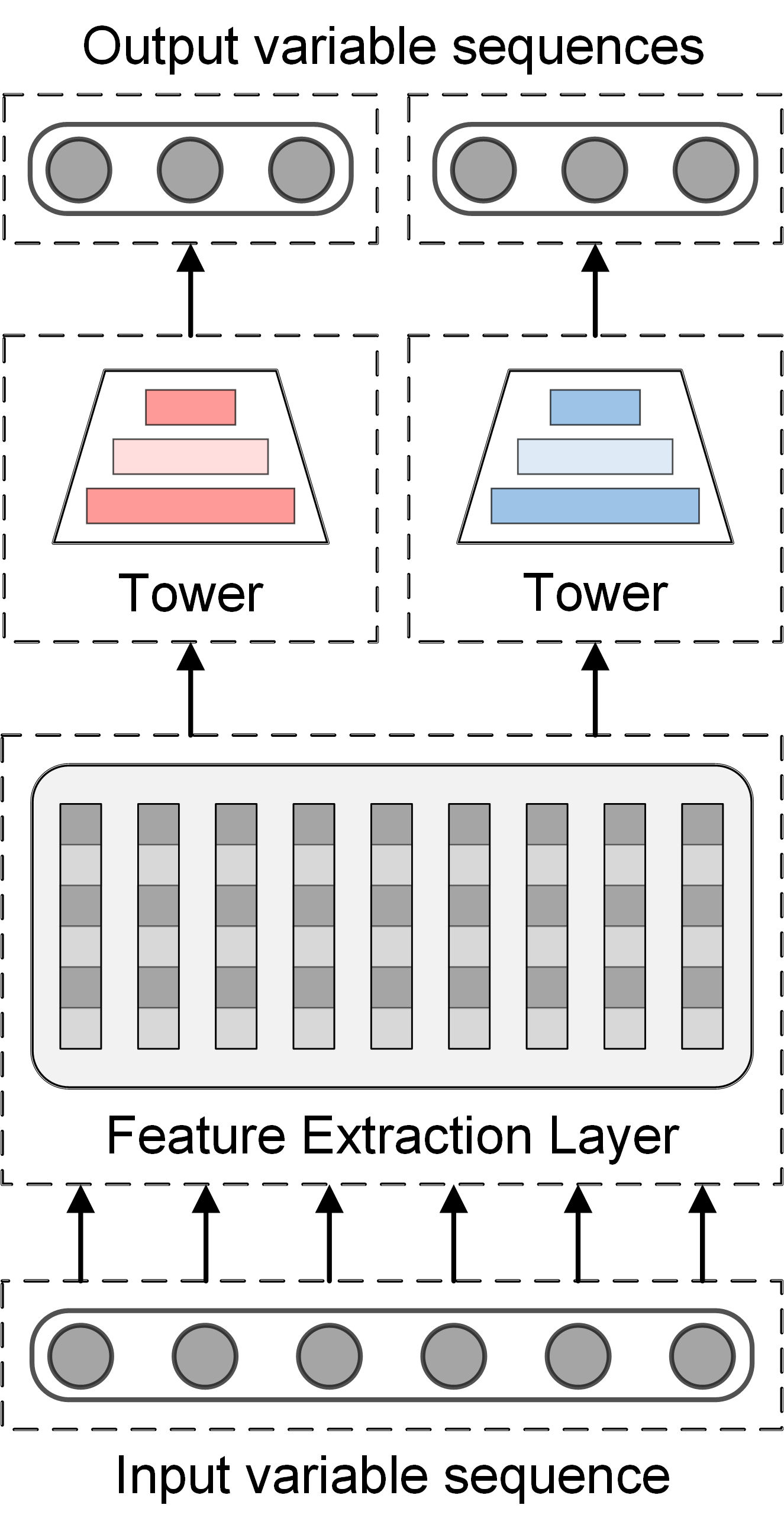}\label{fig:hard-sharing}}
    \subfigure[Cross-Talk]{\includegraphics[width=0.32\columnwidth]{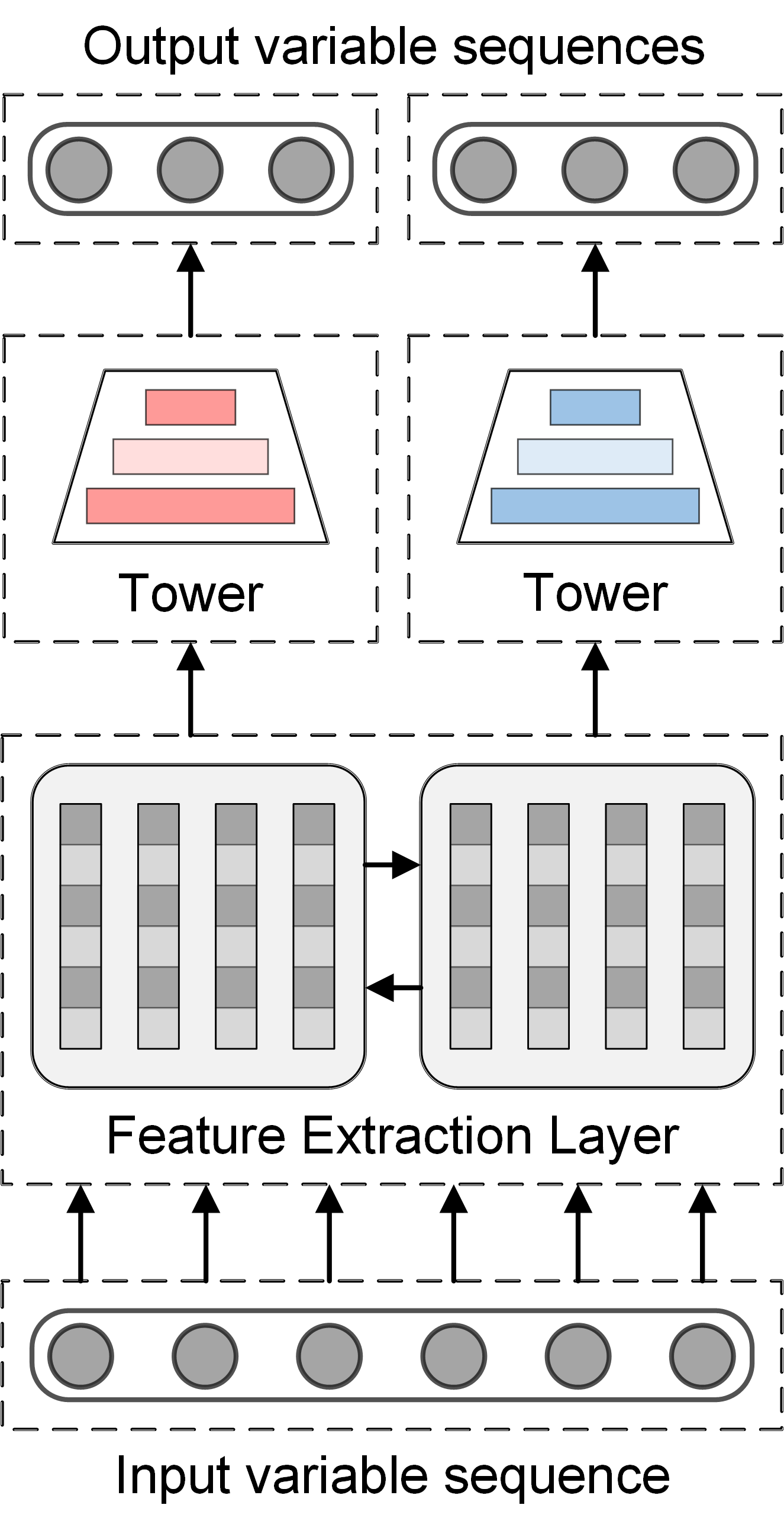}\label{fig:cross-talk}}
    \subfigure[Soft Sharing]{\includegraphics[width=0.32\columnwidth]{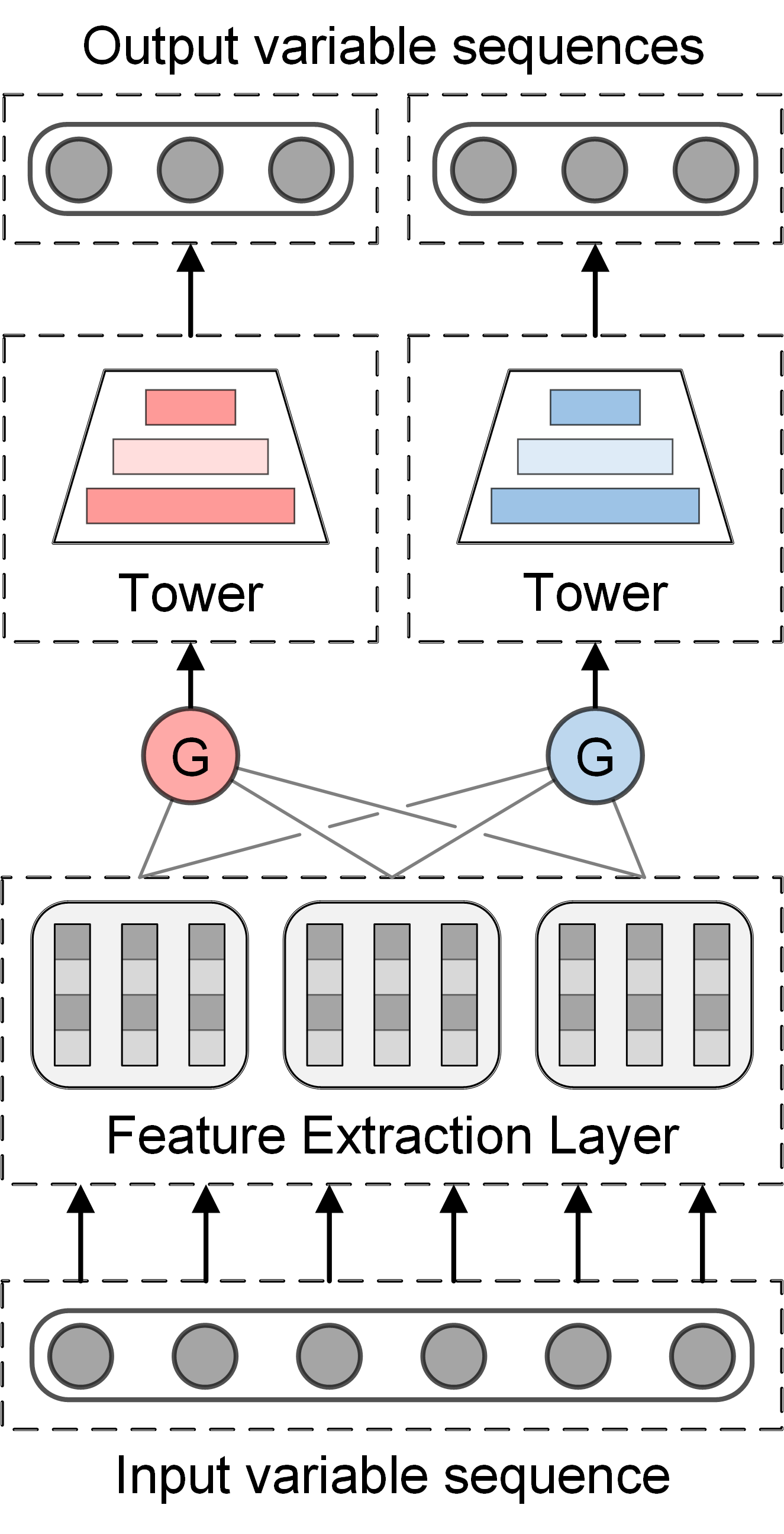}\label{fig:soft-sharing}}
    \caption{Network routing for hard and soft parameter sharing.}
    \vspace{-0.2cm}
\end{figure}

By leveraging both common features shared between objectives and objective-specific features, MOSS is capable of improving efficiency and generalization performance. There are two types of parameter sharing mechanisms that are frequently utilized nowadays. One is hard parameter sharing shown in Fig.~\ref{fig:hard-sharing}, which involves a shared underlying model structure with common hidden layers across objectives. Although this structure minimizes the likelihood of overfitting, optimization conflicts resulting from objective heterogeneity still exists. The other type is soft parameter sharing, which shares experts capable of cross-talk at the bottom layer or fuses experts information via gating networks, as illustrated in Fig.~\ref{fig:cross-talk} and~\ref{fig:soft-sharing}. Compared with hard parameter sharing, soft parameter sharing has more objective-specific parameters and can nevertheless achieve higher performance when objective dependencies are complicated and parameter conflicts arise.

\subsection{Negative Transfer \& Seesaw Phenomenon In MOSS}
NT and seesaw phenomenons are two major issues that must be addressed throughout the MOO process of multiple quality variables in MOSS. Specially, NT phenomenon means that the multi-objective performance attained through training via~\eqref{eq:multi-loss} has deteriorated in comparison to the performance of individual modeling training for each objective. The typical explanation is that there is a weak or loose connection between the multiple objectives of joint modeling or that the MOO model is difficult to learn the similarities and differences between the various objectives.

The seesaw phenomenon usually occurs when the association between multiple objectives is complex. It implies that multiple objectives cannot be performed at a high level at the same time, and that secondary objectives are frequently sacrificed to achieve excellent performance on the main objective.

\subsection{Pareto Optimality For MOSS}
\begin{figure*}[!th]
\vspace{-0.2cm}
\centering
\includegraphics[width=0.9\linewidth]{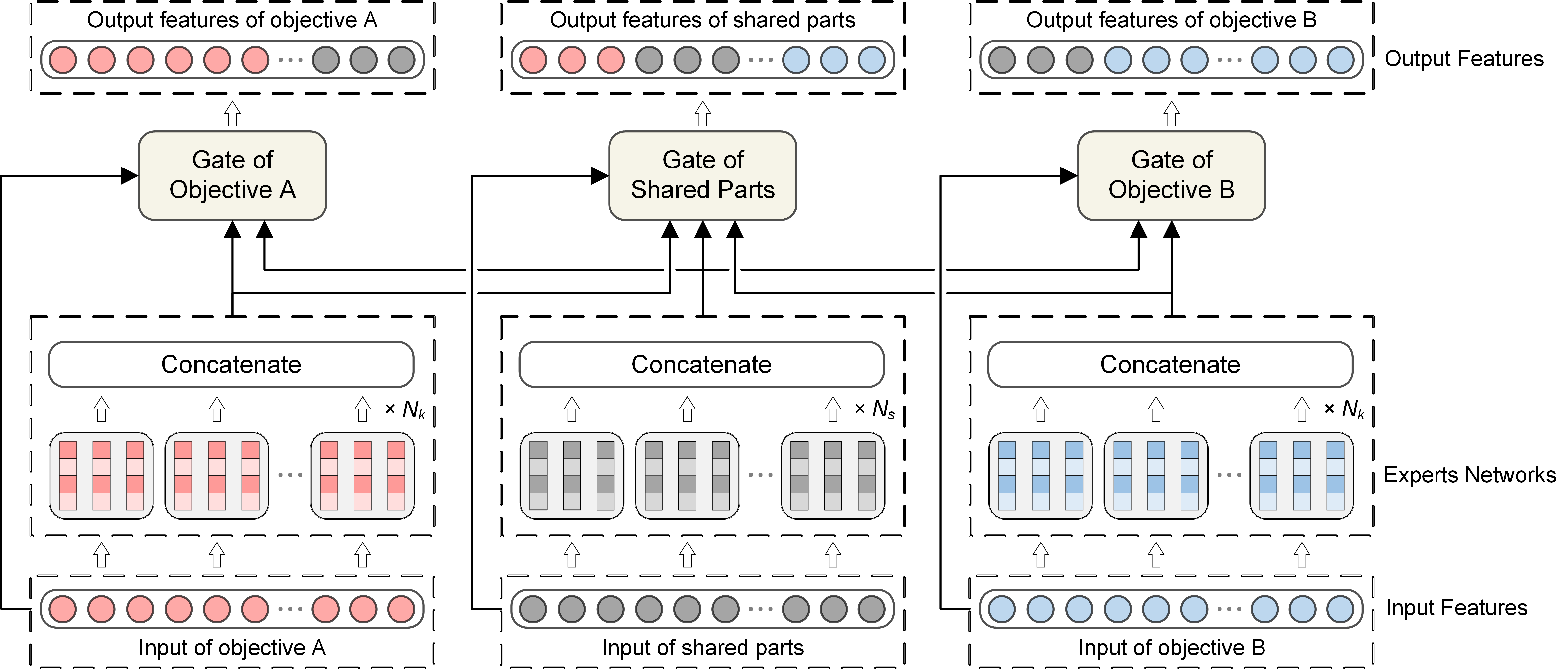}
\caption{The structure of Objective-aware Mixture-of-Experts module.}
\label{fig:omoe}
\vspace{-0.5cm}
\end{figure*}

The MOSS problem specified in~\eqref{eq:pred} tries to obtain estimates $\left\{ \hat{y}_{k} \right\}_{k \in \{1,2,\ldots,K\}}$ for each objective by minimizing the prediction loss of the inference model $f$ on each objective by solving the MOP defined in~\eqref{eq:multi-loss}. However, MOP is tough to solve and acquire the optimal solution. The main reason is that MOP has multiple objective functions, and the superiority and inferiority of MOP solutions cannot be compared and ranked by traditional size relation comparison. To evaluate the correlation between solutions, the Pareto optimality theory is proposed and defined as follows~\citep{mgda} in the parameter sharing mode:
\begin{enumerate}
\item{A parameter solution $\boldsymbol{\theta}$ is said to dominate another parameter solution $\widetilde{\boldsymbol{\theta}}$ if $\hat{\mathcal{L}}_{k}\left(\boldsymbol{\theta}_{\text{sh}}, \boldsymbol{\theta}_{k}\right) \leq \hat{\mathcal{L}}_{k}\left(\widetilde{\boldsymbol{\theta}}_{\text{sh}}, \widetilde{\boldsymbol{\theta}}_{k}\right)$ for all objectives $k$, and there is a objective $k'$ that satisfies $\hat{\mathcal{L}}_{k'}\left(\boldsymbol{\theta}_{\text{sh}}, \boldsymbol{\theta}_{k'}\right) < \hat{\mathcal{L}}_{k'}\left(\widetilde{\boldsymbol{\theta}}_{\text{sh}}, \widetilde{\boldsymbol{\theta}}_{k'}\right)$.}
\item{A parameter solution $\boldsymbol{\theta}^{*}$ is called Pareto optimal if and only if there is no parameter solution $\boldsymbol{\theta}$ dominating $\boldsymbol{\theta}^{*}$.}
\end{enumerate}

In fact, the solution of MOP is not unique, but rather a set of Pareto optimal solutions, often known as Pareto set. Besides, the KKT conditions listed below need to be met while solving MOP:
\begin{enumerate}
\item{There exists $c_{1},c_{2},\ldots,c_{K}$ that satisfies $\sum_{k=1}^{K}c_{k}=1$ and $\sum_{k=1}^{K} c_{k} \nabla_{\boldsymbol{\theta}_{\text{sh}}} \hat{\mathcal{L}}_{k}\left(\boldsymbol{\theta}_{\text{sh}}, \boldsymbol{\theta}_{k}\right)=0$.}
\item{For all objectives $k$, $\nabla_{\boldsymbol{\theta}_{k}} \hat{\mathcal{L}}_{k}\left(\boldsymbol{\theta}_{\text{sh}}, \boldsymbol{\theta}_{k}\right)=0$ is satisfied.}
\end{enumerate}

A solution that meets the above conditions is called a Pareto stationary point, and all Pareto optimal points are Pareto stationary, although the opposite is not always true. Therefore, while training the MOO model, we focus on leveraging Pareto stationary point to approximate the optimal MOO model parameters.

\section{Proposed Method}
\label{sec:proposed-method}
\subsection{Objective-aware Mixture-of-Experts Module}
The OMoE module is a multi-gate hybrid expert module inspired by~\citep{ple} that was designed under the proposed framework. In contrast to the multi-gate hybrid expert models used in~\citep{mmoe} and~\citep{bmoe}, the OMoE module explicitly distinguishes between objective-shared experts and objective-specific experts. Based on the MMoE structure of using objective-specific gate to fuse shared expert information to mitigate the NT phenomenon, the OMoE module increases objective-specific experts to achieve single-objective modeling performance, thereby solving the seesaw phenomenon caused by MOO imbalance.

Specifically, the OMoE module is composed of several objective-specific and objective-shared feature extraction blocks, each with its own set of objective-specific expert module, objective-shared expert module, and corresponding gating networks. Multi-layer networks associated with each objective make up the OMoE module's top layer. The number of objective-shared and objective-specific experts in the feature extraction blocks, as well as the width and depth of the experts networks and tower networks, are all hyperparameters. Fig.~\ref{fig:omoe} shows an example of a single feature extraction block to illustrate data flow in the OMoE module. The inputs of the current feature extraction block are respectively derived from the selected outputs of objective-specific and objective-shared expert module in the previous feature extraction block. For the $j^{\text{th}}$ feature extraction block, the outputs of $k^{\text{th}}$ objective-specific expert module and the objective-shared expert module will be selected and concatenated as follows to obtain features $O_{k}$ and $O_{s}$:
\begin{equation}\label{eq:out-fea}
\begin{aligned}
    O_{k}\left(\boldsymbol{x}^{(k)}\right)&=\left[ \text{cat} \left\{E_{(k, p)}^{\top}\right\}_{p=1}^{n_{k}}, \text{cat} \left\{E_{(s, q)}^{\top}\right\}_{q=1}^{n_{s}}\right]^{\top}, \\
    O_{s}\left(\boldsymbol{x}^{(s)}\right)&=\left[ \text{cat} \left\{\left\{E_{(k, p)}^{\top}\right\}_{p=1}^{n_{k}}\right\}_{k=1}^{K}, \text{cat} \left\{E_{(s, q)}^{\top}\right\}_{q=1}^{n_{s}}\right]^{\top}.
\end{aligned}
\end{equation}
where $\text{cat} \left\{ \boldsymbol{\alpha}^{\top}, \boldsymbol{\beta}^{\top} \right\}=\left[ \boldsymbol{\alpha}^{\top}, \boldsymbol{\beta}^{\top} \right]$ is the concatenation function, $\boldsymbol{x}^{(k)}$ is the input of the $k^{\text{th}}$ objective-specific expert module, $E_{(k, p)}$ is the output features of the $p^{\text{th}}$ expert in the $k^{\text{th}}$ objective-specific expert module, $\boldsymbol{x}^{(s)}$, $E_{(s, q)}$ correspond to the objective-shared expert module, and $n_{k}$, $n_{s}$ are the number of experts in the $k^{\text{th}}$ objective-specific expert module and objective-shared expert module respectively. It should be noted that the corresponding inputs for all modules in the first feature extraction block is the same, that is $\boldsymbol{x}^{(k)}=\boldsymbol{x}^{(s)}$.

\begin{figure}[!t]
\centering
\includegraphics[width=0.85\columnwidth]{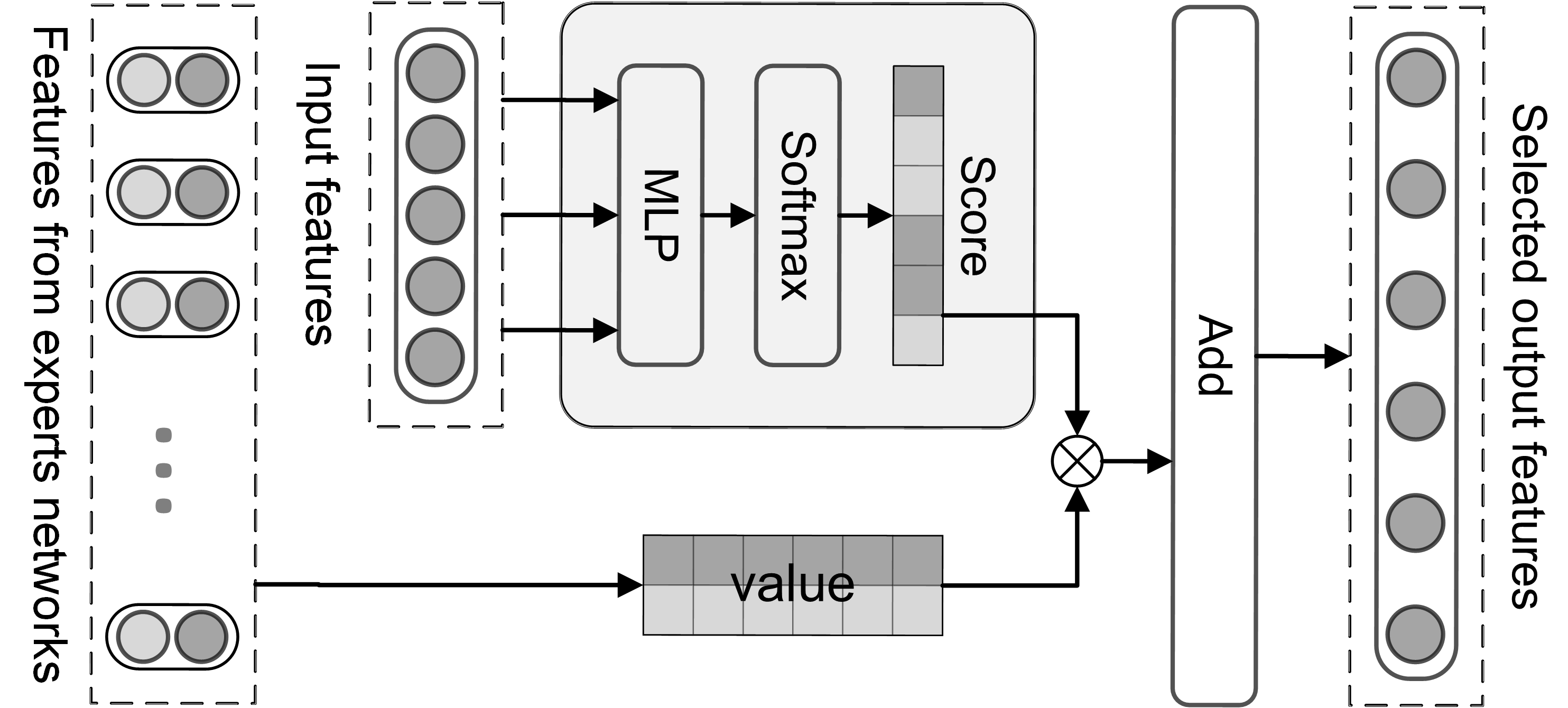}
\caption{The workflow of gating network.}
\label{fig:gating}
\vspace{-0.6cm}
\end{figure}

Following selective concatenation, the output features are selectively fused via gating networks. The gating network is a single-layer fully connected network with softmax as the activation function, and the input is used as a filter to determine the weighted sum of the chosen splicing vectors, as shown in Fig.~\ref{fig:gating}. Taking the $k^{\text{th}}$ objective-specific gating network as an example, its output after feature extraction block (FEB) is also used as the input of the $k^{\text{th}}$ objective-specific expert module in block $j+1$ as follows:
\begin{equation}\label{eq:feb}
    \text{FEB}\left(\boldsymbol{x}^{(k)}\right)=g_{k}\left( \boldsymbol{x}^{(k)} \right)O_{k}\left(\boldsymbol{x}^{(k)}\right).
\end{equation}
where $g_{k}$ is the weight calculation network and also the input linear transformation with a sofmax layer:
\begin{equation}
    g_{k}\left( \boldsymbol{x}^{(k)} \right)=\text{softmax}\left( \boldsymbol{W}_{k}\boldsymbol{x}^{(k)}\right).
\end{equation}
where $\boldsymbol{W}_{k} \in \mathbb{R}^{n \times d}$ is a trainable matrix, $n$ is the number of experts in $O_{k}$ and $d$ is the dimension of input $\boldsymbol{x}^{(k)}$.

Finally, the output of the gating network is fed into the matching tower network to obtain the predicted value:
\begin{equation}
    \hat{y}_{k} = h_{k}\left( g_{k}^{(n_{b})}\left( {\boldsymbol{x}}^{(k,n_{b})} \right) \right).
\end{equation}
where $h_{k}$ is the $k^{\text{th}}$ tower network and $n_{b}$ is the number of feature extraction block.

The OMoE module is a MOO model of soft parameter sharing framework, but unlike MMoE and cross-stitch networks, OMoE divides feature extraction into objective-specific and objective-shared parts. This design adopts separate gating networks to integrate expert information for multiple objectives to mitigate the NT phenomenon, while using objective-specific experts to achieve performance comparable to single-objective modeling to mitigate the seesaw phenomenon.

\subsection{Pareto Objective Routing Module}
The genesis of the seesaw phenomenon is more than just a problem with the architecture of MOO models. Another cause of the uneven distribution of training resources among objectives is the imbalance of loss weights. The optimization direction of the MOO model is determined by the weighted sum of the loss weight and the objective loss, and the incorrect setting of the loss weight may sacrifice the performance of the secondary objectives to improve the performance of the primary objectives. Therefore, the POR module is designed in such a way that the loss weight is dynamically modified as the model parameters are optimized.

\subsubsection{Problem Statement}
As previously stated, the MOP defined by~\eqref{eq:multi-loss} is difficult to have a unique solution to minimize the loss of all objectives. The typical solution is to convert the loss vector optimization in~\eqref{eq:multi-loss} into scalar optimization, which is realized by the weighted sum of the following formula:
\begin{equation}\label{eq:scaler-loss}
    \min_{\boldsymbol{\theta}} \mathcal{L}_{\boldsymbol{w}}\left(\boldsymbol{\theta}\right)=\min_{\boldsymbol{\theta}} \boldsymbol{w}^{\top} \mathcal{L}\left(\boldsymbol{\theta}\right)=\min_{\boldsymbol{\theta}} \sum_{k=1}^{K} w_{k} \hat{\mathcal{L}}_{k}\left(\boldsymbol{\theta}_{\text{sh}}, \boldsymbol{\theta}_{k}\right).
\end{equation}
where $\boldsymbol{w} = \left[w_{1},\ldots,w_{K}\right]^{\top}$ is the loss weight, satisfying $\sum_{k=1}^{K} w_{k}=1$ and $w_{k} \geq 0$ for all objectives $k$.

The Pareto stationary point of the multi-objective loss minimization problem in scalar form can be approximated using the following sub-optimization problem:
\begin{equation}\label{eq:sub-opt}
    \min_{\boldsymbol{w}} L\left( \boldsymbol{w} \right)= \min_{\boldsymbol{w}} \left\Vert \nabla_{\boldsymbol{\theta}_{\text{sh}}} \mathcal{L}_{\boldsymbol{w}} \right\Vert_{2}^{2} .
\end{equation}
where $\Vert\cdot\Vert_{2}$ is the $L_{2}$-norm.

The sub-optimization problem~\eqref{eq:sub-opt} is a linearly constrained optimization problem about the loss weight $\boldsymbol{w}$, and its optimal solution $\boldsymbol{w}^{*}$ can be obtained through linear search algorithm Frank-Wolfe~\citep{fw}.

\begin{table}[!b]
\vspace{-0.3cm}
\centering
\normalsize
\setlength{\tabcolsep}{3pt}
\newcolumntype{R}[1]{>{\raggedleft\let\newline\\\arraybackslash\hspace{0pt}}p{#1}}
\resizebox{\columnwidth}{!}{
\begin{tabular}{@{}R{15pt} p{234pt}@{}}
    \toprule
    \multicolumn{2}{l}{
        \begin{minipage}[t]{\columnwidth} 
            \textbf{Algorithm 1} POR for Parameter Optimization
        \end{minipage}
    } \\
    \midrule
    \multicolumn{2}{l}{
        \begin{minipage}[t]{\columnwidth}
            \textbf{Input}: $\boldsymbol{\theta}$: parameters of the MOO model; $\nabla_{\boldsymbol{\theta}_{\text{sh}}}\hat{\mathcal{L}}_{k}$: gradient of each objective loss to the objective-shared parameters; $\nabla_{\boldsymbol{\theta}_{k}}\hat{\mathcal{L}}_{k}$: gradient of each objective loss to the objective-specific parameters.
        \end{minipage}
    } \\
    \multicolumn{2}{l}{
        \begin{minipage}[t]{\columnwidth} 
            \textbf{Parameter}: $R$: the maximum number of iterations; $\boldsymbol{e}_{\hat{k}}$: a $K$-dimensional unit vector whose $\hat{k}^{\text{th}}$ element is 1.
        \end{minipage}
    } \\
    \multicolumn{2}{l}{
        \begin{minipage}[t]{\columnwidth} 
            \textbf{Output}: $\boldsymbol{w}^{*}$: loss weight which minimizes $L\left( \boldsymbol{w} \right)$; $\boldsymbol{\theta}$: optimized parameters.
        \end{minipage}
    } \\
    1: & Initialize  $\boldsymbol{w}^{(0)} = \left[w_{1},\ldots,w_{K}\right]^{\top}=\left[\frac{1}{T},\ldots,\frac{1}{T}\right]^{\top}$  \\
    2: & Compute $\boldsymbol{M}$ with element $\boldsymbol{M}_{i,j}=\left(\nabla_{\boldsymbol{\theta}_{\text{sh}}} \hat{\mathcal{L}}_{i}\right)^{\top}\left(\nabla_{\boldsymbol{\theta}_{\text{sh}}} \hat{\mathcal{L}}_{j}\right)$ \\
    3: & Set $r=0$ \\
    4: & \textbf{repeat} \\
    5: & \setlength{\parindent}{12pt} $\hat{k} \leftarrow \text{argmin}_{i} \sum_{j=1}^{K}w_{j}\boldsymbol{M}_{i,j}$ \\
    6: & \setlength{\parindent}{12pt} $v^{*} \leftarrow \text{argmin}_{v} \left[v\boldsymbol{e}_{\hat{k}}+(1-v)\boldsymbol{w}\right]^{\top}\boldsymbol{M}\left[v\boldsymbol{e}_{\hat{k}}+(1-v)\boldsymbol{w}\right]$ \\
    7: & \setlength{\parindent}{12pt} $\boldsymbol{w} \leftarrow v^{*}\boldsymbol{e}_{\hat{k}}+(1-v^{*}) \boldsymbol{w}$\\
    8: & \setlength{\parindent}{12pt} $r=r+1$ \\
    9: & \textbf{until} $v^{*}\rightarrow0$ or $r=R$ \\
    10: & Update $\boldsymbol{\theta}_{k} \leftarrow \boldsymbol{\theta}_{k}-\eta \nabla_{\boldsymbol{\theta}_{k}} \hat{\mathcal{L}}_{k}\left(\boldsymbol{\theta}_{\text{sh}}, \boldsymbol{\theta}_{k}\right)$ for all objectives \\
    11: & Update $\boldsymbol{\theta}_{\text{sh}} \leftarrow \boldsymbol{\theta}_{\text{sh}}-\eta \sum_{k=1}^{K} w_{k} \nabla_{\boldsymbol{\theta}_{\text{sh}}} \hat{\mathcal{L}}_{k}\left(\boldsymbol{\theta}_{\text{sh}}, \boldsymbol{\theta}_{k}\right)$ \\
    \bottomrule
\end{tabular}}
\label{algo:por}
\end{table}

\subsubsection{Frank-Wolfe Algorithm}
The Frank-Wolfe algorithm is a method for approximately minimizing the objective function in the feasible zone after first-order linearization. The loss function $L\left( \boldsymbol{w} \right)$ in~\eqref{eq:sub-opt} can be linearized as follows:
\begin{equation}\label{eq:linear-opt}
    \min_{\boldsymbol{w}} L\left( \boldsymbol{w} \right) \rightarrow \min_{\boldsymbol{w}} \nabla_{\boldsymbol{w}} L\left( \boldsymbol{w}^{(r)} \right)^{\top} \boldsymbol{w}.
\end{equation}
where $\boldsymbol{w}^{(r)}$ is the approximation point in current iteration. 

In the case of two objectives,~\eqref{eq:sub-opt} can be generalized as $\min_{w} \Vert w\boldsymbol{l}_{1}+(1-w)\boldsymbol{l}_{2}\Vert_{2}^{2}$, where $\boldsymbol{l}_{1}, \boldsymbol{l}_{2}$ are the loss gradient vectors of the two objectives relative to the objective-shared parameters. The Frank-Wolfe algorithm points out that the weight $w^{*}$ that makes the norm in the preceding formula the smallest is:
\begin{equation}\label{eq:w-star}
    w^{*}=\left\{
    \begin{array}{lc}
        1 & \text{if} \ \boldsymbol{l}_{1}^{\top}\boldsymbol{l}_{2} \geq \boldsymbol{l}_{1}^{\top}\boldsymbol{l}_{1} \\
        0 & \text{if} \ \boldsymbol{l}_{1}^{\top}\boldsymbol{l}_{2} \geq \boldsymbol{l}_{2}^{\top}\boldsymbol{l}_{2} \\
        \frac{(\boldsymbol{l}_{2}-\boldsymbol{l}_{1})^{\top} \boldsymbol{l}_{2}}{\Vert\boldsymbol{l}_{1}-\boldsymbol{l}_{2}\Vert_{2}^{2}} & \text{otherwise}
    \end{array}
    \right. .
\end{equation}

For the common situation where the number of objectives is higher than two, the objective gradient interconnection matrix $\boldsymbol{M}$ should be firstly defined as follows:
\begin{equation}
    \boldsymbol{M}_{i,j}=\left(\nabla_{\boldsymbol{\theta}_{\text{sh}}} \hat{\mathcal{L}}_{i}\left(\boldsymbol{\theta}_{\text{sh}}, \boldsymbol{\theta}_{i}\right)\right)^{\top}\left(\nabla_{\boldsymbol{\theta}_{\text{sh}}} \hat{\mathcal{L}}_{j}\left(\boldsymbol{\theta}_{\text{sh}}, \boldsymbol{\theta}_{j}\right)\right).
\end{equation}

Further calculate the weighted row sum of $\boldsymbol{M}$ to find the objective $\hat{t}$ that minimizes the weighted sum of its own gradient multiplied by the gradients of other objectives. The weight $v^{*}$ that makes $\left[v\boldsymbol{e}_{\hat{k}}+(1-v)\boldsymbol{w}^{(r)}\right]^{\top}\boldsymbol{M}\left[v\boldsymbol{e}_{\hat{k}}+(1-v)\boldsymbol{w}^{(r)}\right]$ the least can be derived by~\eqref{eq:w-star}. Finally, the loss weight $\boldsymbol{w}$ is updated as follows:
\begin{equation}
    \boldsymbol{w}^{(r+1)} \leftarrow v^{*}\boldsymbol{e}_{\hat{k}}+(1-v^{*}) \boldsymbol{w}^{(r)}.
\end{equation}
where $r$ is the number of iterations and $\boldsymbol{e}_{\hat{k}} \in \mathbb{R}^{K}$ is a unit vector whose $\hat{k}^{\text{th}}$ element is 1.

Iteratively update the loss weight $\boldsymbol{w}$ until $v^{*}$ approaches zero, and then find the loss weight $\boldsymbol{w}^{*}$ that eventually meets the normalization and non-negative conditions. Ultimately, the objective-specific parameters are updated in the same way as the regular network parameters, with the objective-shared parameters updated using the loss weights:
\begin{equation}\label{eq:update-param}
\begin{aligned}
    \boldsymbol{\theta}_{k}& \leftarrow \boldsymbol{\theta}_{k}-\eta \nabla_{\boldsymbol{\theta}_{k}} \hat{\mathcal{L}}_{k}\left(\boldsymbol{\theta}_{\text{sh}}, \boldsymbol{\theta}_{k}\right) \quad k\in\{1,\ldots,K\}, \\
    \boldsymbol{\theta}_{\text{sh}}& \leftarrow \boldsymbol{\theta}_{\text{sh}}-\eta \sum_{k=1}^{K} w_{k} \nabla_{\boldsymbol{\theta}_{\text{sh}}} \hat{\mathcal{L}}_{k}\left(\boldsymbol{\theta}_{\text{sh}}, \boldsymbol{\theta}_{k}\right).
\end{aligned}
\end{equation}
where $\eta$ is the learning rate.

Algorithm 1 displays the POR module for parameter optimization of the MOO model.

\subsubsection{Analysis of Primal-Dual Convergence}
We shall demonstrate the primal-dual convergence of the Frank-Wolfe algorithm, which is stronger than the convergence of the primal problem, using a simple dual gap proof framework. The convergence of the primal error is given first in Theorem~\ref{thm:primal-conv}.

\begin{thm}\label{thm:primal-conv}
For each $r \ge 1$, the iterates $\boldsymbol{w}^{(r)}$ of the Frank-Wolfe algorithm satisfy
\begin{equation}
    L(\boldsymbol{w}^{(r)})-L(\boldsymbol{w}^{*}) \leq \frac{2 C_{f}}{r+2}(1+\delta).
\end{equation}
where $\delta \ge 0$ is the accuracy to which the linear subproblems~\eqref{eq:linear-opt} are solved, and $C_{f}$ is the curvature constant of loss function $L$ which is defined as
\begin{equation}
    C_{f}=\sup_{\gamma \in[0,1]} \frac{2}{\gamma^{2}}(L(\boldsymbol{w}')-L(\boldsymbol{w})-\langle\boldsymbol{w}'-\boldsymbol{w}, \nabla L(\boldsymbol{w})\rangle).
\end{equation}
where $\boldsymbol{w}'=\boldsymbol{w}+\gamma(\boldsymbol{s}-\boldsymbol{w})$ and $\boldsymbol{s}$ is a feasible solution.
\end{thm}

The proof of the above convergence theorem is given in Appendix for completeness. Theorem~\ref{thm:primal-conv} guarantees a small raw error, while the optimal value $L(\boldsymbol{w}^{*})$ and the curvature constant $C_{f}$ are usually unknown. The surrogate duality gap is defined below for more convenient estimation of the approximation quality
\begin{equation}
    \phi(\boldsymbol{w})=\max \ \langle\boldsymbol{w}-\boldsymbol{s}, \nabla L(\boldsymbol{w})\rangle.
\end{equation}

Convexity of $L$ implies that the linearization $L(\boldsymbol{w})+\langle\boldsymbol{w}-\boldsymbol{s}, \nabla L(\boldsymbol{w})\rangle$ always lies below the function $L$, which provides the property of the duality gap, $\phi(\boldsymbol{w}) \ge L(\boldsymbol{w})-L(\boldsymbol{w}^{*})$. 

The Frank-Wolfe algorithm obtain guaranteed small duality gap $\phi(\boldsymbol{w}) \le \epsilon$ after $O(\frac{1}{\epsilon})$ iterations even if the linear subproblems~\ref{eq:linear-opt} are only solved approximately.

\begin{thm}\label{thm:dual-conv}
If the Frank-Wolfe algorithm is run for $R \ge 2$ iterations, then the algorithm has a bounded duality gap with iterate $\boldsymbol{w}^{(r)}$, $1 \le r \le R$
\begin{equation}
    \phi(\boldsymbol{w}^{(r)}) \leq \frac{2 \beta C_{f}}{R+2}(1+\delta).
\end{equation}
where $\beta = \frac{27}{8}$.
\end{thm}

By proving that the duality gap cannot be kept large in multiple iterations, Theorem~\ref{thm:dual-conv} is proved by contradiction in Appendix. 

\subsection{Architecture \& Procedure}
\begin{figure}[!t]
\centering
\includegraphics[width=\columnwidth]{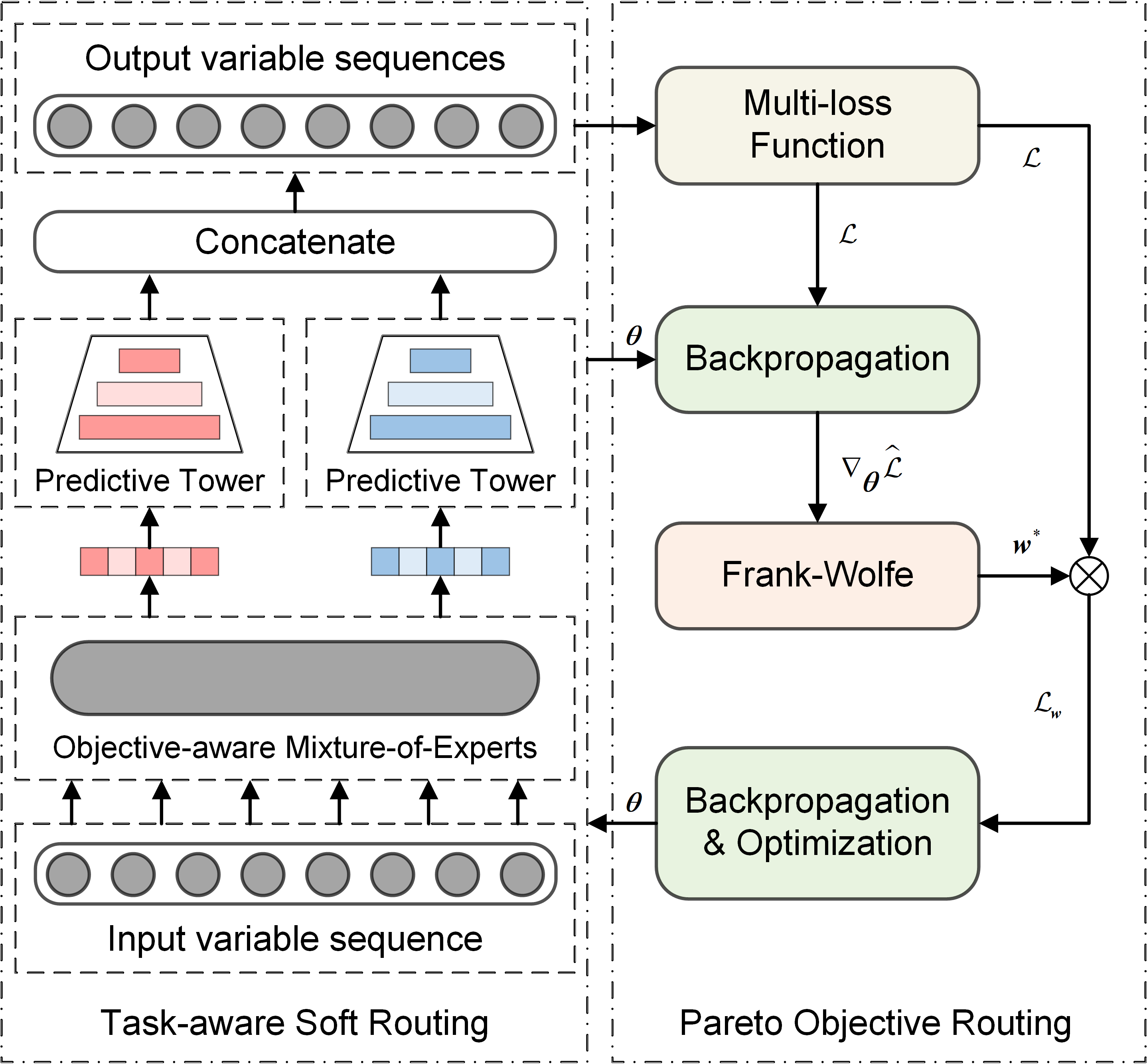}
\caption{The general architecture of the proposed Task-aware Mixture-of-Experts for achieving the Pareto optimum.}
\label{fig:general-framework}
\vspace{-0.5cm}
\end{figure}

Fig.~\ref{fig:general-framework} depicts the general architecture of the proposed TMoE-P model. The model is divided into two components: the OMoE module and the POR module. The input is first routed through the OMoE module, where objective-specific and objective-shared features are progressively separated using stacked feature extraction blocks. After that, the gating networks screen and fuse the separated features through~\eqref{eq:out-fea} and~\eqref{eq:feb}. Additionally, the output of the last feature extraction block will disgard the part of the objective-shared gating network and estimate the quality variables via the tower networks.

The predicted quality variables sequence is utilized in conjunction with the real values to calculate the loss of each objective, while backpropagation is employed to derive the gradient of each objective loss for the objective-shared parameters. The gradient will be submitted to the POR module, which will thereafter iteratively search for the loss weight $\boldsymbol{w}^{*}$ that minimizes~\eqref{eq:sub-opt}. The multi-objective loss $\mathcal{L}_{\boldsymbol{w}}$ is calculated after $\boldsymbol{w}^{*}$ and the loss of each objective have been weighted. Finally, the network parameters are optimized using~\eqref{eq:update-param} to approach the Pareto optimal network parameters.
\section{Experiments}
In this section, we carry out extensive experiments using real material measurement data to assess the effectiveness of the proposed model.

\subsection{Experiment Setup}
\subsubsection{Dataset}
\begin{figure}[!h]
\vspace{-0.2cm}
\centering
\includegraphics[width=0.95\columnwidth]{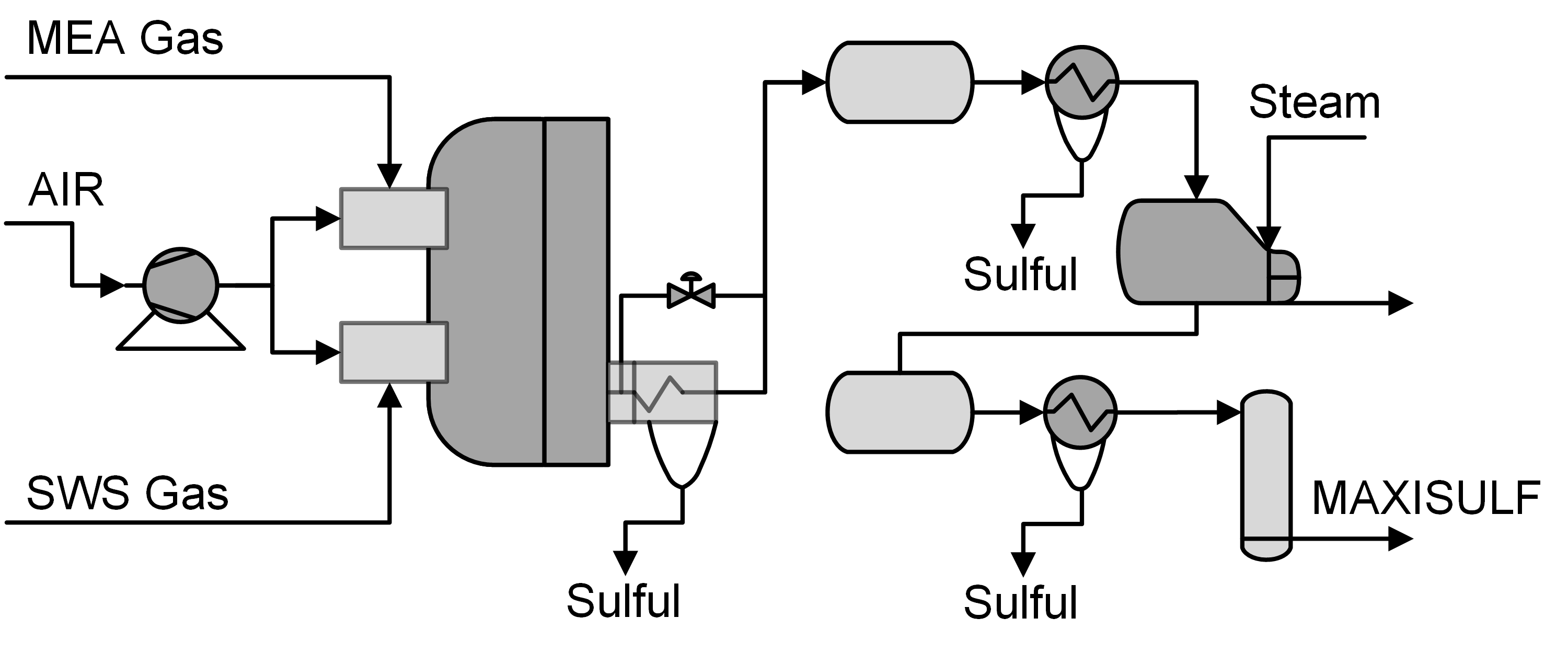}
\caption{Process flow of the SRU.}
\label{fig:sru}
\vspace{-0.2cm}
\end{figure}

The dataset comes from Sulfur Recovery unit (SRU), which is an essential part of managing sulfur emissions~\citep{sru}. Fig.~\ref{fig:sru} depicts an SRU made up of four parallel sulfur production lines, the major inputs of which are $\hhs$-enriched MEA gas and $\hhs$- and $\text{NH}_{3}$-enriched SWS gas.

The key reaction in SRU that can purify acid gas and generate sulfur as a byproduct is as follows:
\begin{equation}\label{eq:sru}
    \left\{
    \begin{array}{ll}
        3 \hhs+\frac{1}{2} \text{O}_{2} & \rightarrow \soo+2 \hhs+\text{H}_{2} \text{O} \\
        \soo+2 \hhs & \rightarrow \text{S}_{x}+2 \text{H}_{2} \text{O}
    \end{array}
    \right. .
\end{equation}

In practical industrial applications, plant operators must manually adjust the air-to-feed ratio such that the concentration ratio of $\hhs$ and $\soo$ is roughly 2:1, where online analyzers are inseparable. However, it's simple for acid gas to harm the instrument itself, hence soft sensors should be considered.

\begin{table}[!t]
\vspace{-0.1cm}
\caption{Process variables and quality variables for SRU}
\centering
\setlength{\tabcolsep}{2.5mm}{
\resizebox{\columnwidth}{!}{
\begin{tabular}{cll}
    \toprule
    Variable & Description  & Unit     \\
    \midrule
    $x_{1}$  & Flow of MEA gas                  & $\text{m}^{3}\text{/s}$     \\
    $x_{2}$  & Flow of air                      & $\text{m}^{3}\text{/s}$     \\
    $x_{3}$  & Flow of secondary air            & $\text{m}^{3}\text{/s}$     \\
    $x_{4}$  & Flow of SWS gas                  & $\text{m}^{3}\text{/s}$     \\
    $x_{5}$  & Flow of air in SWS zone          & $\text{m}^{3}\text{/s}$     \\
    $y_{1}$  & Concentration of $\hhs$ in tail gas & $\text{moles/}\text{m}^{3}$ \\
    $y_{2}$  & Concentration of $\soo$ in tail gas & $\text{moles/}\text{m}^{3}$ \\
    \bottomrule
\end{tabular}}}
\label{tab:variables}
\vspace{-0.6cm}
\end{table}

Table.~\ref{tab:variables} displays the process variables and quality variables collected from the SRU industrial history data. Taking into account the process dynamics and time-delay characteristics, the actual input of the soft sensor model at time $t$ is designed as follows:
\begin{equation}
    \boldsymbol{x}_{t}=\left[ x_{1}(t),\ldots,x_{1}(t-9),\dots,x_{D}(t),\ldots,x_{D}(t-9) \right]^{\top}.
\end{equation}
where $D=5$, $x_{d}(t), d=1,2,\ldots,D$ are the input variables at time $t$, $x_{d}(t-z), z=1,2,\dots,9$ are corresponding lagged values. After removing null values, the dataset comprises a total of 10,000 data samples. 

\subsubsection{Baseline Models}
In experiments, we compare the proposed model to conventional regression prediction models such as partial least squares regression (PLSR), AE, variant of AE like stacked autoencoder (SAE) and gated stacked target-related autoencoder (GSTAE)~\citep{gstae}, long short-term memory network (LSTM), variable attention LSTM (VA-LSTM) network~\citep{va-lstm} and supervised LSTM (SLSTM) network~\citep{slstm}. In addition, we compared SOTA MTL models such as MMoE~\citep{mmoe}, PLE~\citep{ple}, and BMoE~\citep{bmoe}. The appendix provides descriptions of baseline models we used. Since certain regression prediction models can only predict univariate, we undertake a fair assessment of quality variables using univariate modeling. And for regression prediction models capable of multivariate prediction, we substitute their multivariate prediction heads with objective-specific tower networks and transform them into MOO models for comparison.

\subsubsection{Training \& Evaluation}
All quality variable prediction models in experiments are trained and evaluated through MSE loss. The SRU dataset is split into training, validation and test set at a ratio of 6:2:2. In all MOO models and single-objective models, we use RELU-activated MLP networks for each expert or feature extraction layer. The proposed model necessitates the adjustment of several hyperparameters: the number of feature extraction blocks, objective-specific experts, and objective-shared experts. The evaluation metrics employed RMSE, MAE, and $\rr$, and all experiments were carried out on Ubuntu 18.04.2 LTS based on Python 3.7.

\subsection{Performance Evaluation}

\begin{table*}[t]
\vspace{-0.1cm}
\caption{Experimental results on the SRU dataset}
\centering
\resizebox{\linewidth}{!}{
\setlength{\tabcolsep}{2.5mm}{
\begin{tabular}{clcccccc}
    \toprule
    \multicolumn{2}{c}{\multirow{2}{*}{Models}} & \multicolumn{3}{c}{$\hhs$}                          & \multicolumn{3}{c}{$\soo$} \\
    \multicolumn{2}{c}{}       & RMSE           & MAE            & $\rr$             & RMSE           & MAE            & $\rr$ \\
    \midrule
    \multirow{7}{*}{ST} & PLSR & 0.0323$\pm$0.0002 & 0.0187$\pm$0.0002 & 0.6299$\pm$0.0100 & 0.0271$\pm$0.0003 & 0.0191$\pm$0.0001 & 0.7730$\pm$0.0010 \\
                        & AE & 0.0264$\pm$0.0037 & 0.0171$\pm$0.0017 & 0.6950$\pm$0.0936 & 0.0243$\pm$0.0023 & 0.0176$\pm$0.0009 & 0.7833$\pm$0.0456 \\
                        & SAE & 0.0198$\pm$0.0010 & 0.0133$\pm$0.0003 & 0.8008$\pm$0.0228 & 0.0240$\pm$0.0012 & 0.0171$\pm$0.0005 & 0.7916$\pm$0.0131 \\
                        & GSTAE & 0.0226$\pm$0.0019 & 0.0142$\pm$0.0008 & 0.7731$\pm$0.0129 & 0.0247$\pm$0.0012 & 0.0177$\pm$0.0006 & 0.7767$\pm$0.0179 \\
                        & LSTM & 0.0221$\pm$0.0024 & 0.0146$\pm$0.0009 & 0.7719$\pm$0.0127 & 0.0228$\pm$0.0005 & 0.0170$\pm$0.0003 & 0.8045$\pm$0.0178 \\
                        & VA-LSTM & 0.0202$\pm$0.0032 & 0.0131$\pm$0.0009 & 0.8183$\pm$0.0194 & 0.0225$\pm$0.0008 & 0.0170$\pm$0.0004 & 0.8180$\pm$0.0080 \\
                        & SLSTM & 0.0215$\pm$0.0037 & 0.0136$\pm$0.0061 & 0.8124$\pm$0.0674 & \underline{0.0202$\pm$0.0038} & \underline{0.0148$\pm$0.0047} & \underline{0.8568$\pm$0.0488} \\ 
    \midrule
    \multirow{7}{*}{MT} & GSTAE & 0.0226$\pm$0.0020 & 0.0141$\pm$0.0008 & 0.8319$\pm$0.0230 & 0.0249$\pm$0.0012 & 0.0178$\pm$0.0005 & 0.8065$\pm$0.0172 \\
                        & VA-LSTM & 0.0205$\pm$0.0023 & 0.0133$\pm$0.0007 & 0.8351$\pm$0.0140 & 0.0238$\pm$0.0008 & 0.0175$\pm$0.0004 & 0.8091$\pm$0.0098 \\
                        & SLSTM & 0.0245$\pm$0.0073 & 0.0193$\pm$0.0074 & 0.8051$\pm$0.1166 & 0.0247$\pm$0.0051 & 0.0164$\pm$0.0066 & 0.8024$\pm$0.0852 \\
                        & MMoE & 0.0225$\pm$0.0018 & 0.0147$\pm$0.0008 & 0.8444$\pm$0.0163 & 0.0235$\pm$0.0018 & 0.0173$\pm$0.0004 & 0.8292$\pm$0.0267 \\
                        & BMoE & 0.0207$\pm$0.0013 & 0.0149$\pm$0.0007 & 0.8665$\pm$0.0163 & 0.0222$\pm$0.0006 & 0.0176$\pm$0.0006 & 0.8534$\pm$0.0079 \\
                        & PLE & \underline{0.0192$\pm$0.0010} & \underline{0.0131$\pm$0.0005} & \underline{0.8764$\pm$0.0240} & 0.0225$\pm$0.0012 & 0.0167$\pm$0.0012 & 0.8371$\pm$0.0236 \\
                        & \textbf{Ours} & \textbf{0.0177$\pm$0.0004}* & \textbf{0.0122$\pm$0.0008}* & \textbf{0.8845$\pm$0.0206} & \textbf{0.0207$\pm$0.0023} & \textbf{0.0153$\pm$0.0017} & \textbf{0.8626$\pm$0.0243} \\
    \bottomrule
\end{tabular}}}
\label{tab:exp-on-sru}
\vspace{-0.4cm}
\end{table*}

The baseline approaches and the proposed TMoE-P model's prediction performance evaluation metrics for the SRU dataset are provided in Table.~\ref{tab:exp-on-sru}, together with the mean and variation. And Fig.~\ref{fig:pred-error} depicts the prediction curve, the accompanying residual histogram, and the kernel density estimation curve. 

\begin{figure}[t]
\vspace{-0.15cm}
\setlength{\abovecaptionskip}{-0.2cm}
\centering
\includegraphics[width=1.04\columnwidth]{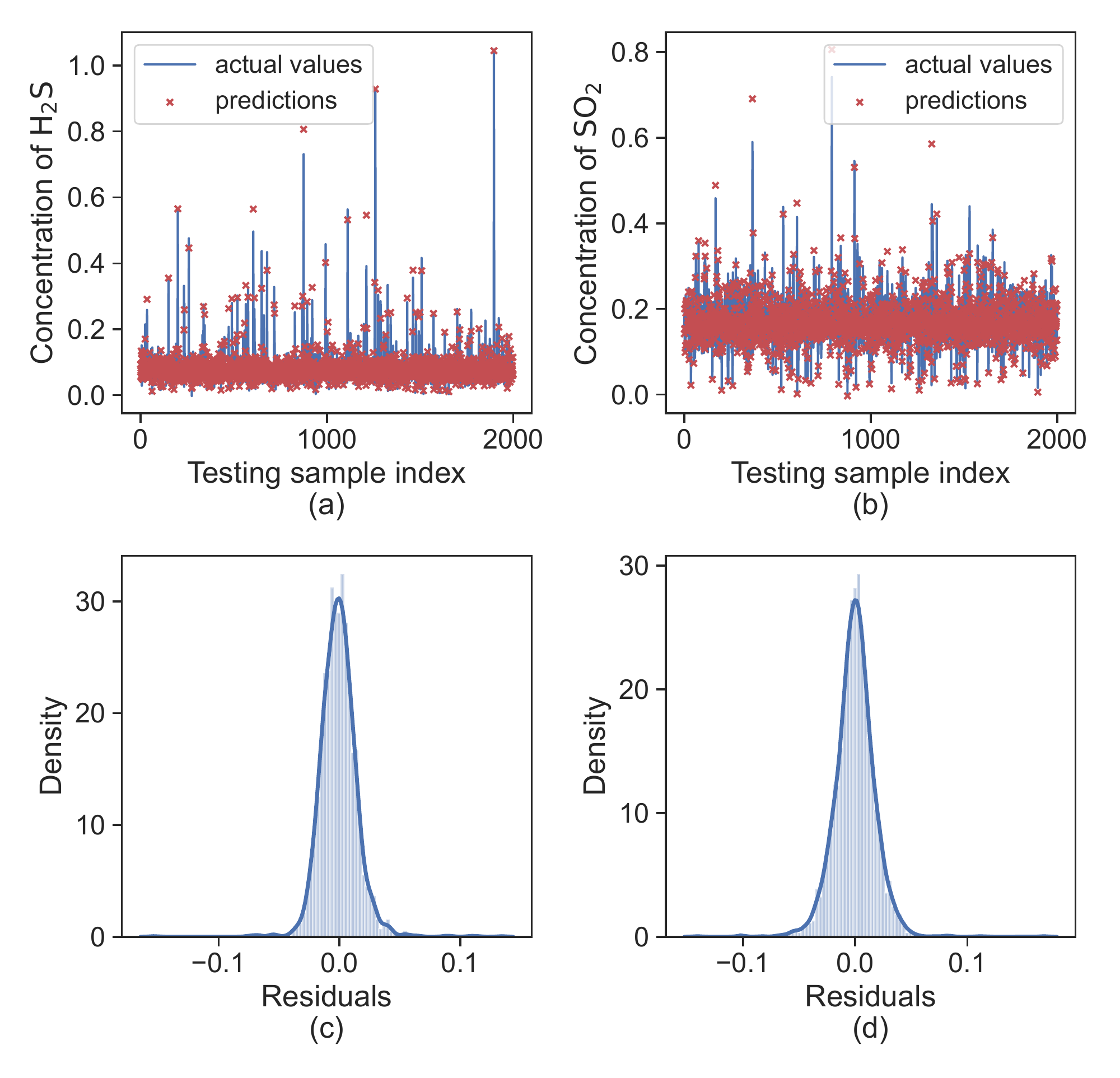}
\caption{Prediction curve, accompanying residual histogram, and the kernel density estimation curve of TMoE-P. (a) and (c) for $\hhs$, (b) and (d) for $\soo$.}
\label{fig:pred-error}
\vspace{-0.5cm}
\end{figure}

The outcomes demonstrate that the proposed TMoE-P model greatly outperforms the baseline models in the concentration prediction of both $\hhs$ and $\soo$. As seen from SRU's reaction process~\eqref{eq:sru}, there is a inverse and complex relationship between between the concentration of $\hhs$ and $\soo$, leading to the occurence of NT phenomenon. Because the dependencies between objectives are challenging to describe when multi-objective joint modeling is required for a single-objective regression prediction model, this is notably evident in the SLSTM. In single-objective modeling of $\soo$, the SLSTM model scored the greatest metric among most of the baseline approaches; however, when employed in multi-objective modeling, the $\rr$ of $\soo$ concentration prediction declined by 0.054.

Furthermore, our experimental results reveal that in the multi-objective scenario of $\hhs$ and $\soo$ concentration prediction, the majority of the multi-objective models suffer from seesaw phenomenon. Specifically, several models decrease the accuracy of $\soo$ while improving the accuracy of $\hhs$. Take the VA-LSTM model for example, the $\rr$ of $\hhs$ concentration prediction increased by 0.02 while the $\rr$ of $\soo$ concentration prediction decreased by about 0.01. On this premise, the MMoE and BMoE models utilize the mixture-of-experts structure to balance multiple objectives, and significantly improve $\hhs$ with a slight decrease in $\soo$ performance. The PLE model incorporates objective-specific experts, which improves the accuracy of $\hhs$ concentration prediction while exacerbating the seesaw phenomenon. The proposed model TMoE-P employs objective-aware mixture-of-experts, uses the POR module to dynamically modify target weights, and can effectively extract correlation information between objectives, outperforming or even outperforming single modeling of each objective.

The kernel density estimation curve in Fig.~\ref{fig:pred-error} ensures that the residuals are Gaussian white noise with zero mean and variance concentration, which verifies that TMoE-P is an unbiased estimation model in statistical significance

\begin{figure}[t]
\vspace{-0.4cm}
\setlength{\abovecaptionskip}{-1pt}
\leftline{\subfigure{\includegraphics[width=1.02\columnwidth]{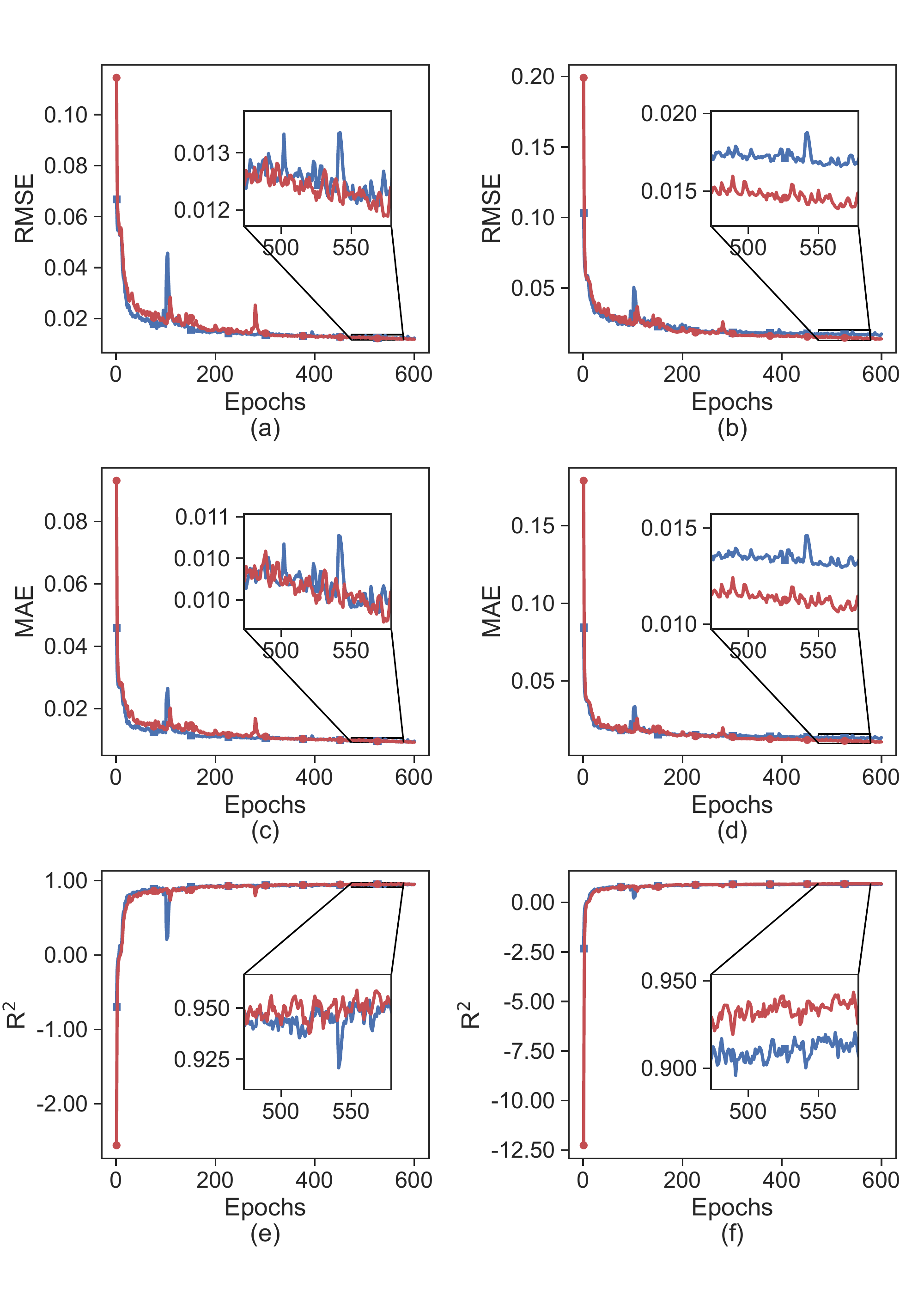}}}
\vskip -0.7cm
\rightline{\subfigure{\includegraphics[width=0.95\columnwidth]{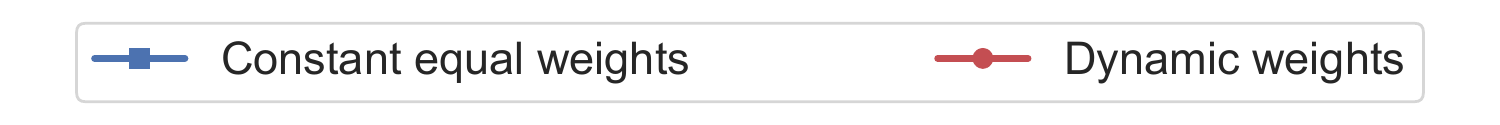}}}
\caption{Performance of dynamic weights and constant equal weights during model training. (a),(c) and (e) are the metrics of $\hhs$ concentration; (b),(d) and (f) are the metrics of $\soo$ concentration.}
\label{fig:metrics-curve}
\vspace{-0.5cm}
\end{figure}

\begin{figure}[t]
\vspace{-0.05cm}
\setlength{\abovecaptionskip}{-0.2cm}
\centering
\includegraphics[width=1.03\columnwidth]{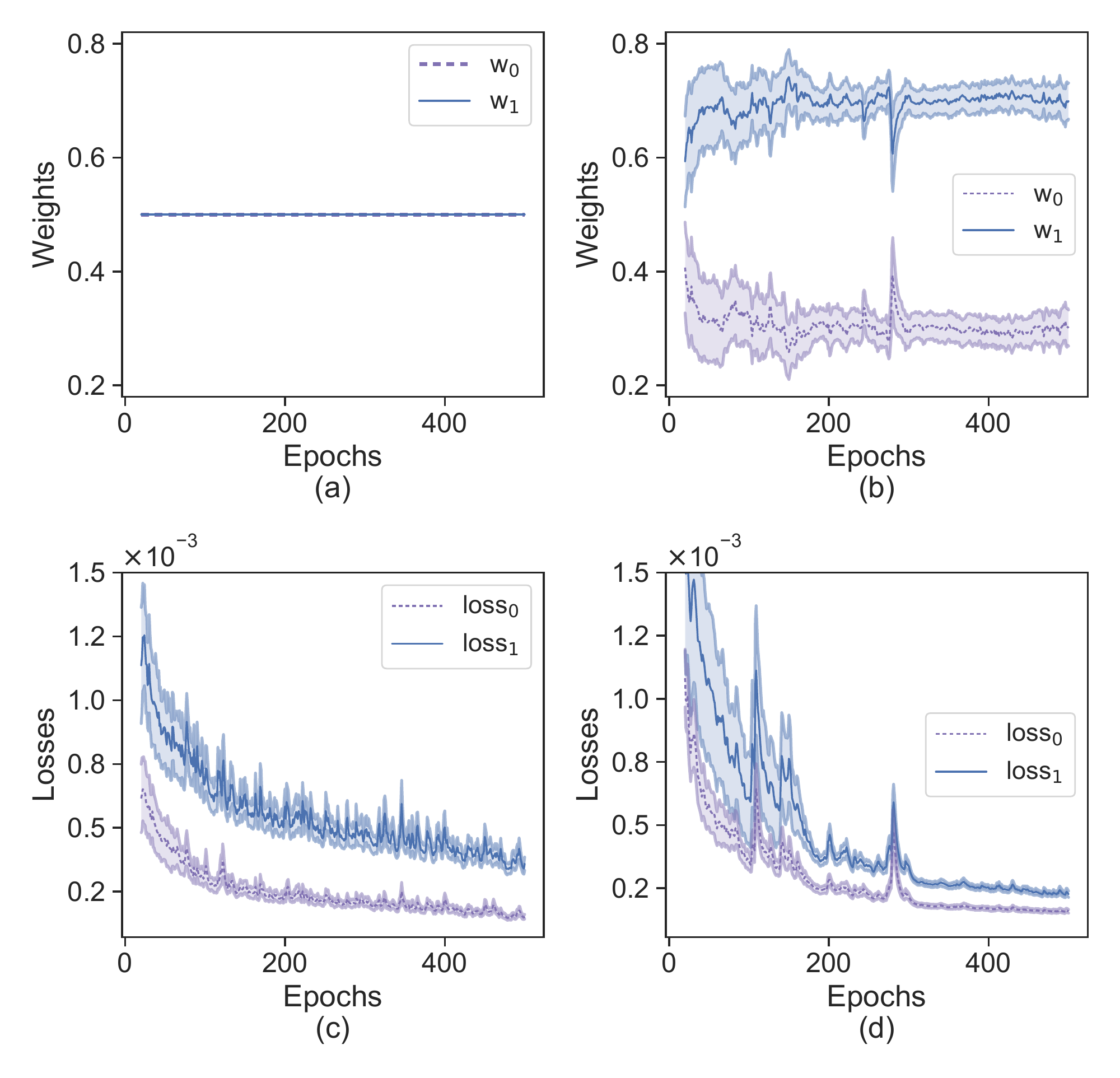}
\caption{Dynamic target weights and corresponding losses during training process of TMoE-P.}
\label{fig:dynamic-weights}
\end{figure}

\begin{figure}[t]
\vspace{-0.2cm}
\setlength{\abovecaptionskip}{-0.15cm}
\centering
\includegraphics[width=1.03\columnwidth]{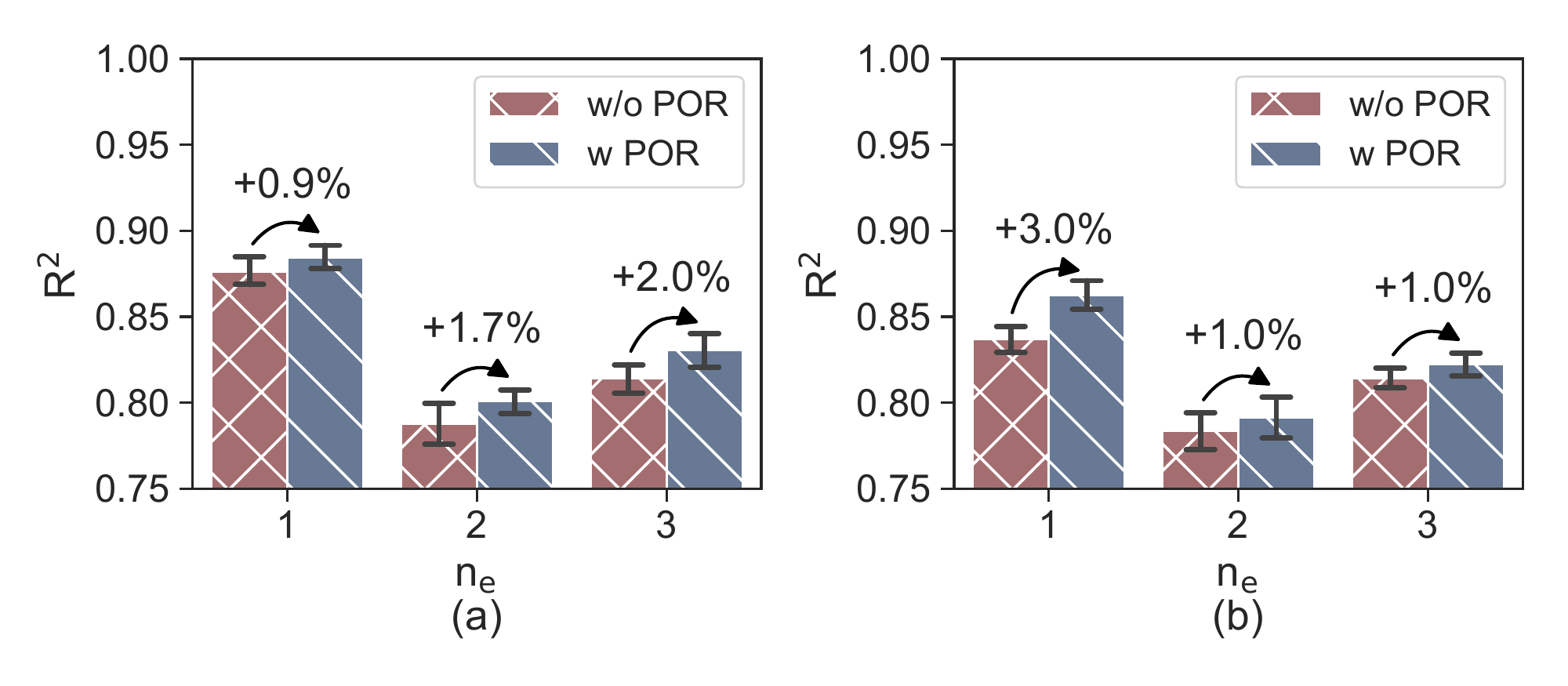}
\caption{Ablation study of POR module under different $n_{e}$. (a) for $\hhs$, and (b) for $\soo$.}
\label{fig:abla-por-omoe}
\vspace{-0.5cm}
\end{figure}

\subsection{Ablation Studies}
\subsubsection{On the POR Module}
The introduction of the POR module to dynamically modify the target weights in the updating of the OMoE network parameters is one of the proposed TMoE-P model's primary advances. This subsection will verify the effectiveness of dynamic weights in the POR module by constrasting the OMoE model trained with constant target weights on the same parameters. 

The trend chart of metrics in the $\hhs$ and $\soo$ concentration prediction objectives of the TMoE-P network trained with constant target weights and dynamic weights is shown in Fig.~\ref{fig:metrics-curve}. It is obvious that the model trained with the POR module's dynamic weights outperforms the model trained with constant weights in terms of steady-state metrics values after convergence. In particular, the $\rr$ of $\hhs$ has risen from 0.8764 to 0.8845 while the $\rr$ of $\soo$ has risen from 0.8371 to 0.8626. Furthermore, dynamic weights outperform single-objective modeling by 8.1\% and 0.67\% respectively, while eliminating the seesaw phenomenon and NT phenomenon. To better understand the evolution trend of dynamic weights, we extracted target weights during the TMoE-P training process, as illustrated in Fig.~\ref{fig:dynamic-weights}. The figure shows that as the objective-specific features and objective-shared features are progressively separated, the target weight evolves from equilibrium to focus on $\hhs$ concentration prediction, eventually settling at roughly 0.7:0.3. The dynamic adjustment of target weights allows the TMoE-P to overcome the imbalance of training resources produced by excessively complicated objective correlations, bringing it closer to the Pareto optimal model, demonstrating the effectiveness of the POR module in MOO.

We further verified the improvement of POR module to OMoE module under different number of experts $n_{e}$, as shown in Fig.~\ref{fig:abla-por-omoe}. It can be seen from the step value in the histogram that, with the same number of experts, the proposed TMoE-P has a higher $\rr$ than OMoE alone, and the variance is generally lower, ensuring more stable and accurate predictions.

In addition, we applied the POR module to the traditional soft sensor models VA-LSTM, SLSTM and the multi-task learning model MMoE, and the results are shown in Fig.~\ref{fig:abla-por-models}. It should be noted that the shared parameter of the first two models is the subject feature extraction network, and the objective-specific parameter is the prediction head. And it can be clearly seen that POR module can produce effective accuracy improvement no matter it is in the traditional soft sensor model or in the multi-task learning model.

\begin{figure}[t]
\vspace{-0.15cm}
\setlength{\abovecaptionskip}{-0.15cm}
\centering
\includegraphics[width=1.03\columnwidth]{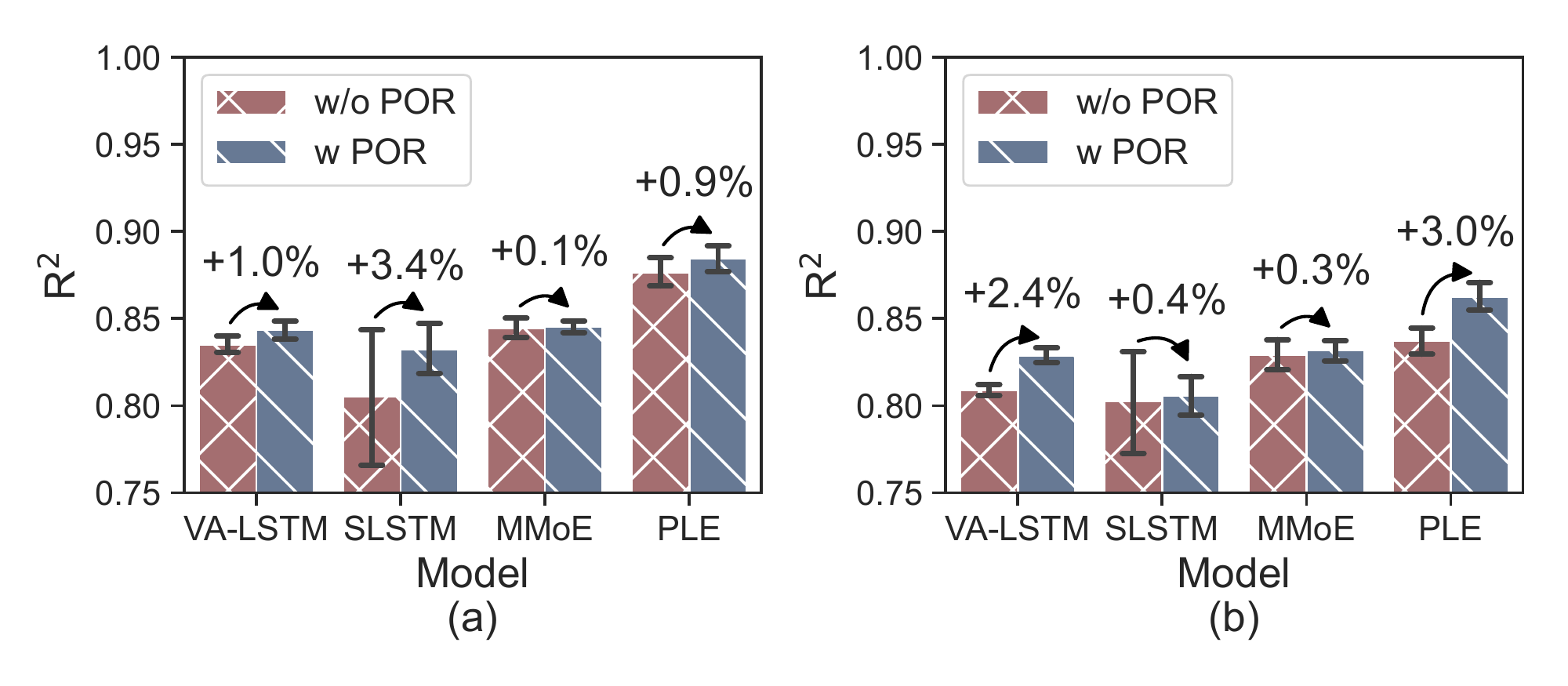}
\caption{Ablation study of POR module under different model. (a) for $\hhs$, and (b) for $\soo$.}
\label{fig:abla-por-models}
\end{figure}

\subsubsection{On the Experts Module}
In Table~\ref{tab:abla-nk-ns}, we remove objective-shared and objective-specific components from the OMoE module to study the differences between target-alone and hard parameter sharing modeling versus hybrid modeling. We find that hybrid modeling, which combines objective-shared and objective-specific experts, is the best MOO framework across all parameter groups. OMoE without objective-shared experts fully ignores the correlation information between quality variables, making it challenging to improve the multi-objective metrics by increasing the number of experts and executing information fusion through the gating network. And regrettably, the $\rr$ metric fell by 2\% $\sim$ 6\% for predictions of two quality variables. The OMoE structure without objective-specific experts, on the other hand, is analogous to the MMoE structure, which cannot escape the seesaw phenomenon in MOO.

\begin{table}[!t]
\vspace{-0.2cm}
\caption{Ablation study of components in TMoE-P}
\centering
\tiny
\setlength{\tabcolsep}{3mm}{
\resizebox{\columnwidth}{!}{
\begin{tabular}{cccc}
    \toprule
    $n_{k}$ & $n_{s}$ & $\hhs-\rr$        & $\soo-\rr$        \\ \midrule
    1  & 0  & 0.8764$\pm$0.0258 & 0.8379$\pm$0.0276 \\
    0  & 1  & 0.8646$\pm$0.0165 & 0.8462$\pm$0.0259 \\
    1  & 1  & \textbf{0.8845$\pm$0.0206} & \textbf{0.8626$\pm$0.0243} \\ \midrule
    2  & 0  & 0.7580$\pm$0.0304 & 0.7833$\pm$0.0245 \\
    0  & 2  & 0.7526$\pm$0.0388 & 0.7841$\pm$0.0218 \\
    2  & 2  & 0.8009$\pm$0.0210 & 0.7914$\pm$0.0358 \\ \midrule
    3  & 0  & 0.8256$\pm$0.0171 & 0.8143$\pm$0.0176 \\
    0  & 3  & 0.8119$\pm$0.0217 & 0.8123$\pm$0.0222 \\
    3  & 3  & 0.8306$\pm$0.0312 & 0.8222$\pm$0.0206 \\ \midrule
    4  & 0  & 0.7622$\pm$0.0378 & 0.7854$\pm$0.0157 \\
    0  & 4  & 0.7954$\pm$0.0273 & 0.7896$\pm$0.0252 \\
    4  & 4  & 0.8045$\pm$0.0418 & 0.7884$\pm$0.0288 \\
    \bottomrule
\end{tabular}}}
\label{tab:abla-nk-ns}
\vspace{-0.4cm}
\end{table}

\begin{figure}[t]
\vspace{-0.5cm}
\setlength{\abovecaptionskip}{-0.2cm}
\centering
\includegraphics[width=1.03\columnwidth]{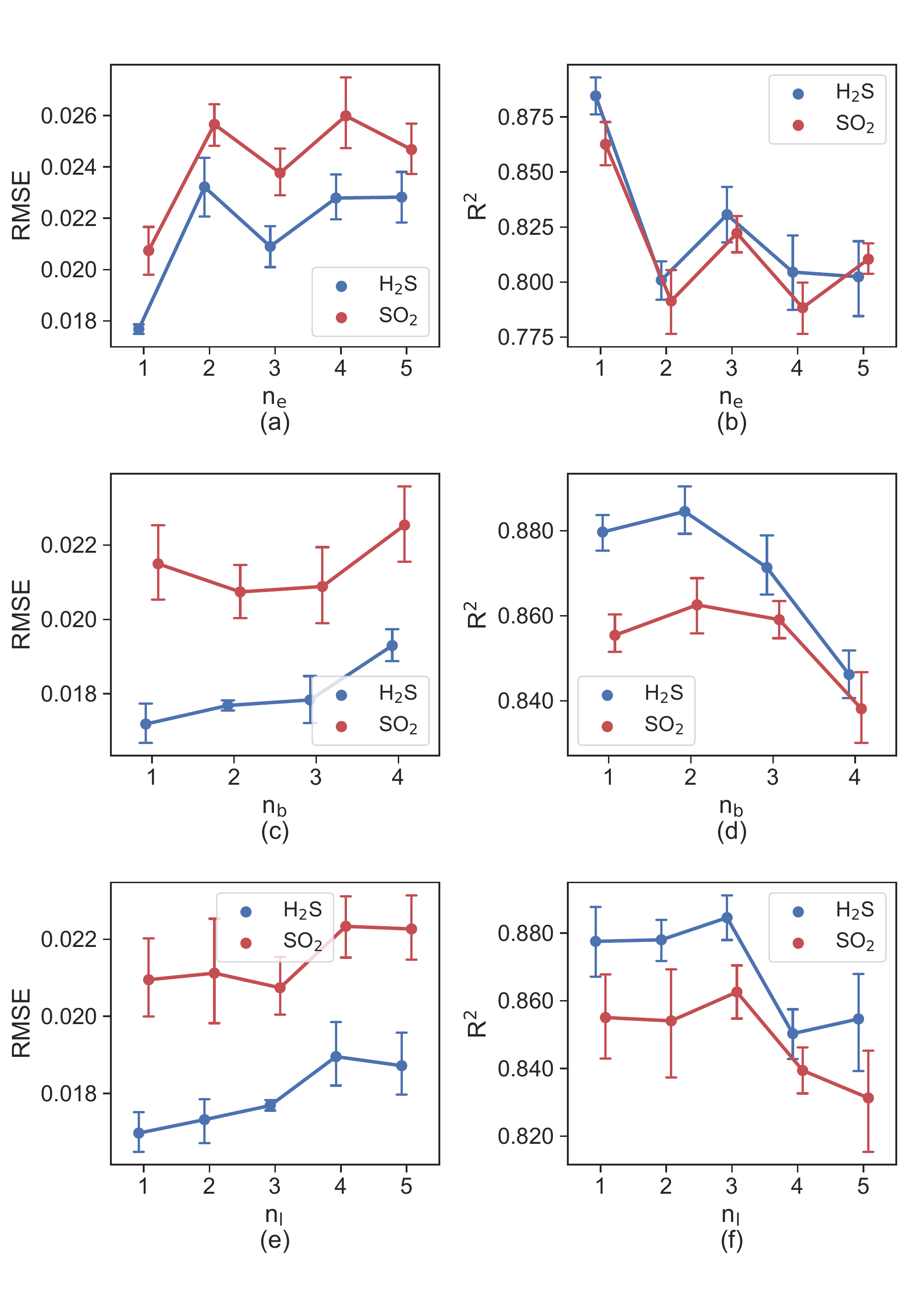}
\caption{Parameter sensitivity analysis of the hyperparameters on metrics RMSE and $\rr$. (a),(c) and (e) for RMSE metric; (b),(d) and (f) for $\rr$ metric.}
\label{fig:sensitivity}
\vspace{-0.4cm}
\end{figure}

\subsection{Parameter Sensitivity Studies}
The main hyperparameters of the proposed TMoE-P model are the number of objective-specific experts $n_{e}$, the number of feature extraction blocks $n_{b}$, and the number of expert network layers $n_{l}$. The sensitivity analysis of the TMoE-P prediction performance will be performed on the three variable hyperparameters in this subsection, with the remaining parameters fixed and one of the above hyperparameters changed through the control variable approach.

Fig.~\ref{fig:sensitivity} depicts the fluctuation diagram of the evaluation metrics RMSE and $\rr$ when $n_{e}$, $n_{b}$ and $n_{l}$ vary between $[1,5]$, $[1,4]$ and $[1,5]$. According to the variations in the two metrics, it can be seen that when $n_{e}=1$, the sum of TMoE-P's metrics in the prediction of $\hhs$ and $\soo$ concentrations is the highest, reaching 0.8845 and 0.8626, respectively. However, when $n_{e}$ grows, the performance of TMoE-P falls dramatically. One plausible explanation is that the gating network shown in Fig. 3 is difficult to integrate large-capacity expert knowledge, and the TMoE-P with single expert is adequate to realize joint prediction of $\hhs$ and $\soo$ concentrations.

When it comes to the effect of $n_{b}$, TMoE-P performs best for $n_{b}=2$, and as $n_{b}$ decreases or grows, TMoE-P performs worse and worse. Indeed, when $n_{b}=1$, TMoE-P can only separate and fuse objective-specific and objective-shared features via a objective-specific gating network, and when $n_{b}$ is too large, the network structure becomes too complex for both to avoid the seesaw phenomenon. Additionally, we discovered that when $n_{l}=3$, the objective-specific expert structure of TMoE-P attained highest performance, outperforming the $n_{l}=2$ and $n_{l}=4$ models in terms of $\rr$, which improved by 0.7\% and 4.0\% for $\hhs$ and 1.0\% and 2.8\% for $\soo$, respectively.

To summarize the sensitivity experiment results for the aforementioned hyperparameters, we recommend that during the TMoE-P training process, first set $n_{e}$ to 1 and then change the hyperparameters $n_{b}$ and $n_{l}$ within $[1,3]$ and $[2,4]$.

\section{Related Work}
Solving industrial soft sensor problems based on the concept of MOO is a novel and demanding task, as there are frequently contradicting but closely connected linkages between quality variables. Broadly speaking, current work in this field can be categorised in two groups. The first is to utilize MOO as the model hyperparameter optimizer to find the best model structure. The second instead generalize and solve the soft sensor problem using the mathematical form of MOO. 

He et al.\citep{relate1} optimized the hyperparameters of the four data-driven models including Random Forest, Gradient Boosting Regression, Ridge Regression, and K-Nearest Neighbor using the intelligent evolutionary algorithm NSGA \Rmnum{2}, but did not further integrate the models. Instead, Jin et al.\citep{relate2} simply performed an optimization search of similarity measures for Just-in-time learning models and developed a stacked Just-in-time learning model with a two-layer structure for online estimation of quality variables. While considering the dimensionality of features and the weight of regularization, Ribeiro et al.\citep{relate3} proposed a two-stage MOO to improve the resilience of Partial Least Squares Regression.

Yan et al.\citep{relate4} suggested an end-to-end multi-quality variable correlation model for joint optimization, with the prediction loss of quality variables and the correlation loss of features serving as the optimization objectives in MOO. However, the nonlinear feature transformation considered in this model is merely Multilayer Perceptron, which is a standard hard parameter sharing model. Huang et al.\citep{bmoe} characterized the soft sensor problem as a multi-task regression problem, which is a special case of MOO, and used the MMoE structure and GradNorm module to balance the task gradient to improve multi-task performance. Futhermore, Lei et al.\citep{relate5} extended the regression task with generation and classification tasks, and employed the upgraded Capsule Network and weighted objective loss for stochastic gradient multi-task optimization.

Our approach is most similar to Huang, in that we both use gradients to create target weights to optimize network parameters. The distinction is that we concentrate primarily on three aspects: we begin by solving the MOP. Compared with multi-task learning, MOO focuses on approximating the Pareto optimal solution, and the network update mechanism is more challenging. Second, the OMoE and MMoE structures use different soft parameter sharing approaches, adding objective-specific experts, which can mitigate the seesaw phenomenon produced by full sharing of the underlying parameters when objective correlation is weak. Finally, our POR optimization algorithm approaches the Pareto optimal model parameters, whereas GradNorm lacks theoretical backing and relies on intuition to discover target weights that balance multi-task learning speed.

\section{Conclusion}
In this paper, we propose a Task-aware Mixture-of-Experts model combining the OMoE and the POR module capable of approaching Pareto optimality (TMoE-P) for the field of industrial soft sensor. The proposed TMoE-P uses the OMoE network to explicitly separate objective-specific parameters from objective-shared parameters in order to achieve soft sensor accuracy that exceeds individual modeling under shared mode. Simultaneously, we employ the POR module, which can approach the Pareto optimality, to dynamically modify the proportion of soft sensor objectives during the training of the OMoE network, avoiding the phenomenon of negative transfer and seesaw. The experimental results on the SRU chemical industry data set reveal that TMoE-P outperforms the SOTA MTL model and the convensional regression prediction model. Our future research will mainly focus on investigating the fusion mechanism of the gating network with a large number of experts, MTL modeling for time-series data in the soft sensor field, and the Pareto stationary point solving algorithm under the MOP optimization framework.


{\appendix[Baseline Description]
The baseline model we used in the verification experiments of TMoE are as follows:
\begin{itemize}
\item{\textbf{PLSR} finds a linear regression model by projecting quality variables and process variables into a new space;}
\item{\textbf{AE} performs representation learning on the input data to obtain dimensionality reduction features, and predicts quality variables after combining with a fully connected neural network;}
\item{\textbf{SAE} increases the depth of the hidden layer based on AE and improves feature extraction ability through layer-wise unsupervised pre-training;}
\item{\textbf{GSTAE}~\citep{gstae} adds the prediction error of the quality variables to the loss function during the layer-wise unsupervised pre-training, guiding the feature learning process with target-related information, and uses gated neurons to learn the information flow of the final output neurons;}
\item{\textbf{LSTM} is a recurrent neural network capable of processing quality variable data arranged in sequences and capturing temporal dependencies between sequences;}
\item{\textbf{VA-LSTM}~\citep{va-lstm} consists of two LSTM networks, one considers the correlation of process variables with quality variables and assigns attention weights to process variables based on the correlations, and the other captures the long-term dependencies of weighted input to predict quality variables;}
\item{\textbf{SLSTM}~\citep{slstm} consists of two LSTM networks, one considers the correlation of process variables with quality variables and assigns attention weights to process variables based on the correlations, and the other captures the long-term dependencies of weighted input to predict quality variables;}
\item{\textbf{MMoE}~\citep{mmoe} is multi-gate mixture-of-experts, which characterizes task correlation and learns specific tasks based on shared representations;}
\item{\textbf{BMoE}~\citep{bmoe} leverages GradNorm on the basis of MMoE to balance the learning gradient between tasks;}
\item{\textbf{PLE}~\citep{ple} explicitly distinguishes between task-specific and task-shared components in the model, avoiding the seesaw phenomenon caused by loose task correlation;}
\end{itemize}

\section*{Proof of Convergence Theorem}
The proof of the convergence of the primal error depends on the following Lemma~\ref{lem:iter-inequality} on the improvement in each iteration.

\begin{lem}\label{lem:iter-inequality}
For a step $\boldsymbol{w}^{(r+1)}=\boldsymbol{w}^{(r)}+\gamma(\boldsymbol{s}-\boldsymbol{w}^{(r)})$ with arbitrary step-size $\gamma \in [0,1]$, if $\boldsymbol{s}$ is a appropriate descent direction on the linear approximation to $L$, it holds that
\begin{equation}
    L\left(\boldsymbol{w}^{(r+1)}\right) \leq L\left(\boldsymbol{w}^{(r)}\right)-\gamma \phi\left(\boldsymbol{w}^{(r)}\right)+\frac{\gamma^{2}}{2} C_{f}(1+\delta).
\end{equation}
\end{lem}

\begin{proof}
To simplify the notation, we write $\boldsymbol{w}=\boldsymbol{w}^{(r)}$, $\boldsymbol{w}'=\boldsymbol{w}^{(r+1)}$, and $d_{w}=\nabla L\left(\boldsymbol{w}^{(r)}\right)$. From the definition of the curvature constant $C_f$, we have
\begin{equation}
\begin{aligned}
    L(\boldsymbol{w}') & =L(\boldsymbol{w}+\gamma(\boldsymbol{s}-\boldsymbol{w})) \\
    & \leq L(\boldsymbol{w})+\gamma\left\langle\boldsymbol{s}-\boldsymbol{w}, d_{w}\right\rangle+\frac{\gamma^{2}}{2} C_{f}.
\end{aligned}
\end{equation}

Then we pick $\boldsymbol{s}$ as an linear minimizer for $L$, and it satisfies $\left\langle\boldsymbol{s}, \nabla L\left(\boldsymbol{w}^{(r)}\right)\right\rangle \leq \min \left\langle\hat{\boldsymbol{s}}, \nabla L\left(\boldsymbol{w}^{(r)}\right)\right\rangle+\frac{1}{2} \delta \gamma C_{f}$, which can be rewritten as
\begin{equation}
\begin{aligned}
    \left\langle \boldsymbol{s}-\boldsymbol{w}, d_{w}\right\rangle & \leq \min \left\langle\boldsymbol{w}', d_{w}\right\rangle-\left\langle\boldsymbol{w}, d_{w}\right\rangle+\frac{1}{2} \delta \gamma C_{f} \\
    & =-\phi(\boldsymbol{w})+\frac{1}{2} \delta \gamma C_{f}.
\end{aligned}
\end{equation}

Therefore, we obtain $L(\boldsymbol{w}') \leq L(\boldsymbol{w})-\gamma \phi(\boldsymbol{w})+\frac{\gamma^{2}}{2} C_{f}(1+\delta)$, which proves the lemma.
\end{proof}

\begin{thm}\label{thm:primal-conv-apdx}
For each $r \ge 1$, the iterates $\boldsymbol{w}^{(r)}$ of the Frank-Wolfe algorithm satisfy
\begin{equation}
    L(\boldsymbol{w}^{(r)})-L(\boldsymbol{w}^{*}) \leq \frac{2 C_{f}}{r+2}(1+\delta).
\end{equation}
where $\delta \ge 0$ is the accuracy to which the linear subproblems~\eqref{eq:linear-opt} are solved, and $C_{f}$ is the curvature constant of loss function $L$ which is defined as
\begin{equation}
    C_{f}=\sup_{\gamma \in[0,1]} \frac{2}{\gamma^{2}}(L(\boldsymbol{w}')-L(\boldsymbol{w})-\langle\boldsymbol{w}'-\boldsymbol{w}, \nabla L(\boldsymbol{w})\rangle).
\end{equation}
where $\boldsymbol{w}'=\boldsymbol{w}+\gamma(\boldsymbol{s}-\boldsymbol{w})$ and $\boldsymbol{s}$ is a feasible solution.
\end{thm}

\begin{proof}
From Lemma~\ref{lem:iter-inequality} we know that $L\left(\boldsymbol{w}^{(r+1)}\right) \leq L\left(\boldsymbol{w}^{(r)}\right)-\gamma \phi\left(\boldsymbol{w}^{(r)}\right)+\gamma^{2}C$ holds for every step of Frank-Wolfe algorithm where we define $C=\frac{C_{f}}{2}(1+\delta)$ and take fixed step $\gamma=\frac{2}{r+2}$. 

Writing the primal error as $\zeta(\boldsymbol{w})=L(\boldsymbol{w})-L(\boldsymbol{w}^{*})$ at any point $\boldsymbol{w}$, this implies that
\begin{equation}\label{eq:primal-error-inequality}
\begin{aligned}
    \zeta\left(\boldsymbol{w}^{(r+1)}\right) & \leq \zeta\left(\boldsymbol{w}^{(r)}\right)-\gamma \phi\left(\boldsymbol{w}^{(r)}\right)+\gamma^{2} C \\
    & \leq \zeta\left(\boldsymbol{w}^{(r)}\right)-\gamma \zeta\left(\boldsymbol{w}^{(r)}\right)+\gamma^{2} C \\
    & =(1-\gamma) \zeta\left(\boldsymbol{w}^{(r)}\right)+\gamma^{2} C.
\end{aligned}
\end{equation}
where we have used weak duality $\zeta(\boldsymbol{w}) \leq \phi(\boldsymbol{w})$. We will now use induction over $r$ to prove bound claimed as follows
\begin{equation}
    \zeta\left(\boldsymbol{w}^{(r+1)}\right) \leq \frac{4 C}{r+3} \quad r=0,1, \ldots
\end{equation}

The base-case $r = 0$ follows from~\eqref{eq:primal-error-inequality} applied for the first step through $\gamma=\gamma^{(0)}=\frac{2}{0+2}=1$. Now considering $r \ge 1$,
\begin{equation}
\begin{aligned}
    \zeta\left(\boldsymbol{w}^{(r+1)}\right) & \leq (1-\gamma^{(r)}) \zeta\left(\boldsymbol{w}^{(r)}\right)+{\gamma^{(r)}}^{2} C \\
    & = (1-\frac{2}{r+2}) \zeta\left(\boldsymbol{w}^{(r)}\right)+{\frac{2}{r+2}}^{2} C \\
    & \leq (1-\frac{2}{r+2}) \frac{4C}{r+2}+{\frac{2}{r+2}}^{2} C .
\end{aligned}
\end{equation}

Simply rearranging the terms gives the bound claimed above
\begin{equation}
\begin{aligned}
    \zeta\left(\boldsymbol{w}^{(r+1)}\right) & \leq \frac{4C}{r+2}(1-\frac{1}{r+2})=\frac{4C}{r+2} \frac{r+1}{r+2} \\
    & \leq \frac{4C}{r+2} \frac{r+2}{r+3} = \frac{4C}{r+3} .
\end{aligned}
\end{equation}
\end{proof}

\begin{thm}
If the Frank-Wolfe algorithm is run for $R \ge 2$ iterations, then the algorithm has a bounded duality gap with iterate $\boldsymbol{w}^{(r)}$, $1 \le r \le R$
\begin{equation}
    \phi(\boldsymbol{w}^{(r)}) \leq \frac{2 \beta C_{f}}{R+2}(1+\delta).
\end{equation}
where $\beta = \frac{27}{8}$.
\end{thm}

\begin{proof}
We do not prove that the above duality gap holds for the entire iteration, but for the last third of the $R$ iteration. To simplify notation, we denote the primal and dual errors as $\zeta^{(r)}=\zeta\left(\boldsymbol{w}^{(r)}\right)$ and $\phi^{(r)}=\phi\left(\boldsymbol{w}^{(r)}\right)$ for $r \ge 0$.

By primal convergence Theorem~\ref{thm:primal-conv-apdx}, we already know that the primal error satisfies $\zeta^{(r)} \leq \frac{C}{r+2}$, where we reset $C=2C_f (1+\delta)$.

We firstly make following contradict assumption: $\phi^{(r)}$ always stays larger than $\frac{\beta C}{U}$ in the last third of the $R$ iterations, which can be formally defined as
\begin{equation}\label{eq:contradict-assumption}
    \phi^{(r)}>\frac{\beta C}{U} \quad \text { for } \quad r \in\{\lceil\mu U\rceil-2, \ldots, R\}
\end{equation}
where $U=R+2$ for simple notation, and $0 < \mu < 1$ is an arbitrary fixed parameter. We will find later that $\mu=\frac{2}{3}$ is a good choice.

Lemma~\ref{lem:iter-inequality} can be further read as follows if we choose $\gamma=\frac{2}{r+2}$
\begin{equation}
\begin{aligned}
    \zeta^{(r+1)} & \leq \zeta^{(r)}-\frac{2}{r+2} \phi^{(r)}+\frac{2}{(r+2)^{2}} C_{f}(1+\delta) \\
    & =\zeta^{(r)}-\frac{2}{r+2} \phi^{(r)}+\frac{C}{(r+2)^{2}}.
\end{aligned}
\end{equation}

Considering the assumptions we made in~\eqref{eq:contradict-assumption}, we obtain
\begin{equation}\label{eq:zeta-inequality-1}
    \zeta^{(r+1)} < \zeta^{(r)}-\frac{2}{r+2} \frac{\beta C}{U}+\frac{C}{(r+2)^{2}}.
\end{equation}

If we define $r_{\text{min}}=\lceil\mu U\rceil-2$, then $r_{\text{min}} \ge 0$ for $R \ge \frac{2(1-\mu)}{\mu}$. While the steps $r$ satisfy $r_{\text{min}} \leq r \leq R$, then $\mu U \leq r+2 \leq U$, the inequality in~\ref{eq:zeta-inequality-1} now reads as
\begin{equation}
\begin{aligned}
    \zeta^{(r+1)} & < \zeta^{(r)}-\frac{2}{U} \frac{\beta C}{U}+\frac{C}{(\mu U)^{2}} \\
    & = \zeta^{(r)}-\frac{2\beta C-C/\mu^{2}}{U^{2}}.
\end{aligned}
\end{equation}

We now sum up this inequality over the last third of the steps from $r_{\text{min}}$ up to $R$, then we get
\begin{align}
\begin{split}
    \zeta^{(R+1)} & < \zeta^{(r_{\text{min}})}-(R-r_{\text{min}}+1)\frac{2\beta C-C/\mu^{2}}{U^{2}} \\
    & \leq \frac{C}{\mu U}-\tau \frac{2\mu\beta-1/\mu}{U} \frac{C}{\mu U} \\
    & = \frac{C}{\mu U} \left( 1-\tau\frac{2\mu\beta-1/\mu}{U} \right).
\end{split}
\end{align}
where $\tau=(1-\mu)U \leq R+2-(\lceil\mu U\rceil-1)=R-r_{\text{min}}+1$, and in the last inequality we have used Theorem~\ref{thm:primal-conv-apdx} giving $\zeta^{(r_{\text{min}})} \leq \frac{C}{r_{\text{min}}+2} \leq \frac{C}{\mu U}$. For $\mu=\frac{2}{3}$ and $\beta=\frac{27}{8}$, the following term become zero: $1-\tau\frac{2\mu\beta-1/\mu}{U} = 1-(1-\mu)(2\mu\beta-1/\mu)=0$. Then we arrive at the contradiction that $\zeta^{(R+1)}<0$ and our assumption on the gap is refuted, so the claimed bound has been proven.
\end{proof}
}

\bibliography{references}
\bibliographystyle{IEEEtran}



\vfill

\end{document}